\documentclass{article}

\usepackage{lmodern}
\usepackage{times}
\usepackage{soul}
\usepackage{url}
\usepackage[utf8]{inputenc}
\usepackage[small]{caption}
\usepackage{graphicx}
\usepackage{amsmath,amssymb,amsthm}
\usepackage{amsfonts}       
\usepackage{mathtools}
\usepackage{bm} 
\usepackage{booktabs}
\usepackage{algorithm}
\usepackage{algorithmic}
\usepackage[switch]{lineno}
\usepackage{multirow}
\usepackage{multicol}
\usepackage{subfigure}
\usepackage{booktabs}
\usepackage{makecell}
\usepackage{xcolor}

\usepackage[accepted]{icml2026}

\theoremstyle{plain}
\newtheorem{theorem}{Theorem}[section]

\theoremstyle{definition}

\newtheorem{assumption}[theorem]{Assumption}
\theoremstyle{remark}

\usepackage[textsize=tiny]{todonotes}

\icmltitlerunning{Mitigating Membership Inference in Intermediate Representations with Differentially Private Training}

\begin{document}

\twocolumn[
  \icmltitle{Mitigating Membership Inference in Intermediate Representations with Differentially Private Training}



  \icmlsetsymbol{equal}{*}

  \begin{icmlauthorlist}
    \icmlauthor{Jiayang Meng}{yyy}
    \icmlauthor{Tao Huang}{comp}
    \icmlauthor{Chen Hou}{comp}
    \icmlauthor{Guolong Zheng}{comp}
    \icmlauthor{Hong Chen}{yyy}
    
  \end{icmlauthorlist}

  \icmlaffiliation{yyy}{School of Information, Renmin University of China, Beijing, China}
  \icmlaffiliation{comp}{School of Computer Science and Big Data, Minjiang University, Fuzhou, Fujian, China}

  \icmlkeywords{Machine Learning, ICML}

  \vskip 0.3in
]



\printAffiliationsAndNotice{}  

\begin{abstract}
  In Embedding-as-an-Interface (EaaI) settings, pre-trained models are queried for Intermediate Representations (IRs). The distributional properties of IRs can leak training-set membership signals, enabling Membership Inference Attacks (MIAs) whose strength varies across layers. Although Differentially Private Stochastic Gradient Descent (DP-SGD) mitigates such leakage, existing implementations employ per-example gradient clipping and a uniform, layer-agnostic noise multiplier, ignoring heterogeneous layer-wise MIA vulnerability. This paper introduces Layer-wise MIA-risk-aware DP-SGD (LM-DP-SGD), which adaptively allocates privacy protection across layers in proportion to their MIA risk. Specifically, LM-DP-SGD trains a shadow model on a public shadow dataset, extracts per-layer IRs from its train/test splits, and fits layer-specific MIA adversaries, using their attack error rates as MIA-risk estimates. Leveraging the cross-dataset transferability of MIAs, these estimates are then used to reweight each layer's contribution to the globally clipped gradient during private training, providing layer-appropriate protection under a fixed noise magnitude. We further establish theoretical guarantees on both privacy and convergence of LM-DP-SGD. Extensive experiments show that, under the same privacy budget, LM-DP-SGD reduces the peak IR-level MIA risk while preserving utility, yielding a superior privacy-utility trade-off.
\end{abstract}

\section{Introduction}
In Embedding-as-an-Interface (EaaI) settings, pre-trained models are exposed as APIs that transform inputs into Intermediate Representations (IRs), which are then consumed by downstream services. This paradigm appears in multiple deployment modes: (i) knowledge distillation \cite{heo2019comprehensive, liu2023function,gong2025beyond}, where a high-capacity, typically frozen teacher model is queried with student’s inputs to produce intermediate embeddings that serve as “hints” for student learning; (ii) Embedding-as-a-Service \cite{thakur2021beir, muennighoff2022mteb}, where cloud-hosted embedding APIs return IRs that downstream modules for retrieval, classification, or clustering directly consume, without exposing model internals; and (iii) modular system design \cite{houlsby2019parameter, wu2024reft}, where a frozen encoder generates IRs that are processed by lightweight adapters or task-specific heads, enabling rapid task adaptation and scalable deployment.

However, exposing IRs from pre-trained models can inadvertently leak training-set membership information. Model optimization tends to concentrate training examples ('members') into high-density regions of the hidden-state manifold, producing systematic separations from non-training examples ('non-members') \cite{zhu2021geometric, wu2024you}. As a result, an adversary with query access to IRs can exploit these separations to launch Membership Inference Attacks (MIAs) \cite{shokri2017membership, nasr2019comprehensive, hu2022membership}, threatening data confidentiality.

Notably, prior work indicates that the susceptibility of IRs to MIAs is inherently layer-dependent, with deeper layers typically encoding stronger membership signals \cite{rezaei2021difficulty,maini2023can,wang2024localizing}. This behavior arises from the hierarchical nature of feature extraction: shallow layers tend to capture generic, low-level features with weaker membership signals, whereas deeper layers encode task-specific semantics that amplify the separation between members and non-members. Consequently, MIA risk varies across layers, and a model’s overall vulnerability to MIAs is governed by the most susceptible layer.

\begin{figure*}[t]
\centering
  \includegraphics[width=0.96\textwidth]{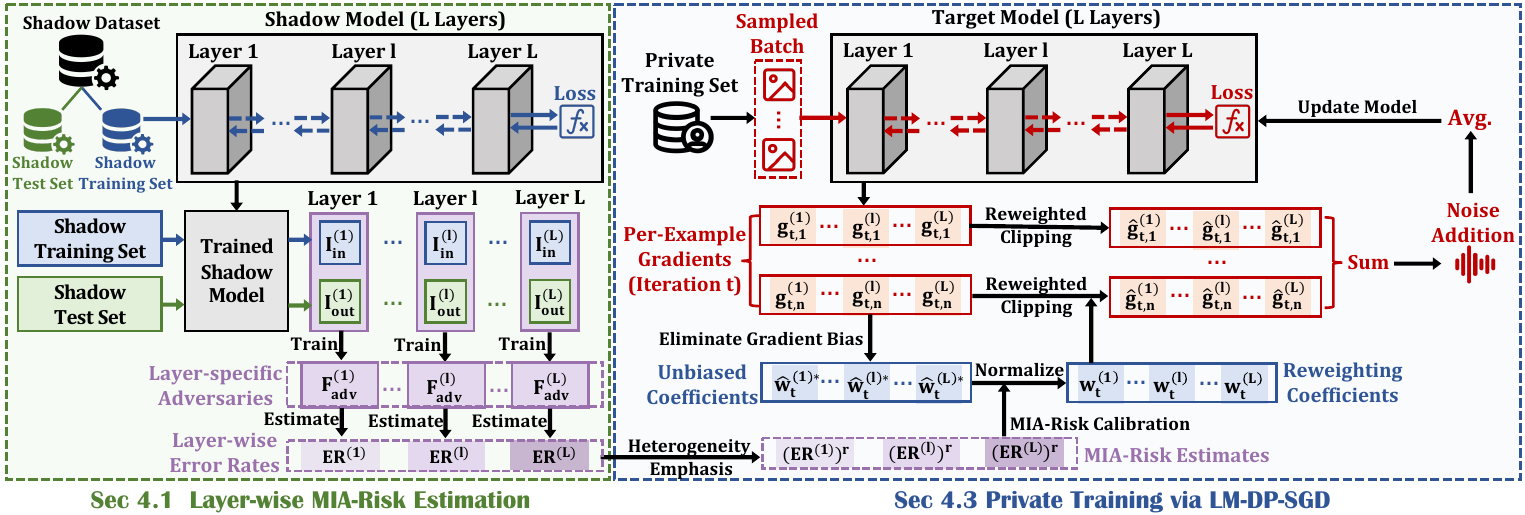}
  \caption{An overview of LM-DP-SGD, which comprises two components: (i) Layer-wise MIA-risk estimation (Section \ref{4.1}), which trains layer-specific adversaries on a public shadow dataset to assess MIA risk of each layer; and (ii) Private training via LM-DP-SGD (Section \ref{4.2}), a differentially private optimization procedure that leverages these estimated risks to perform layer-wise reweighted clipping to per-example gradients before noise injection, allocating protection proportionally to layer vulnerability.}
  \label{fig:oveview}
\end{figure*}

The predominant defense, Differentially Private Stochastic Gradient Descent (DP-SGD) \cite{abadi2016deep}, ignores layer-wise heterogeneity in MIA risk. Current DP-SGD methods enforce $(\varepsilon,\delta)$-Differential Privacy (DP) \cite{dwork2006calibrating, dwork2008differential} by clipping per-example gradients and adding Gaussian noise uniformly across layers. Such uniform perturbation can over-noise less vulnerable layers while under-protecting more sensitive ones, leading to a suboptimal privacy–utility trade-off \cite{bagdasaryan2019differential}. A natural layer-wise variant is to privatize each layer independently; however, composition theorems \cite{kairouz2015composition} incur substantial privacy budget overhead, making such methods impractical for deep models.

To address these limitations, we propose Layer-wise MIA-risk-aware DP-SGD (LM-DP-SGD), a novel DP-SGD variant that adaptively tailors privacy protection across layers in proportion to their MIA risks, as shown in Figure \ref{fig:oveview}. Our contributions can be summarized as follows.

\begin{itemize}
    \item We introduce LM-DP-SGD, a differentially private training framework that (i) trains layer-specific adversaries on a public shadow dataset to estimate per-layer MIA risk; and (ii) uses these estimates to guide layer-wise reweighted clipping of per-example gradients before noise addition, thus modulating the protective effect of fixed-magnitude noise across layers.
    \item  We theoretically analyze the privacy and convergence properties of LM-DP-SGD. We prove that it preserves the same DP guarantees as DP-SGD. We also characterize its convergence behavior and quantify the bias induced by layer-wise reweighted clipping.
    \item Extensive experiments demonstrate that, compared to the baselines, LM-DP-SGD markedly reduces the peak IR-level MIA risk while preserving model utility, thus achieving an improved privacy–utility trade-off.
\end{itemize}

\section{Related Work}
\subsection{Gradient Clipping Strategies in DP-SGD}

DP-SGD \cite{abadi2016deep}, a cornerstone of privacy-preserving deep learning, modifies the mini-batch Stochastic Gradient Descent (SGD) with per-example gradient clipping and calibrated Gaussian noise injection to ensure formal DP guarantees. Given a per-example gradient $g$, clipping is defined as $\text{Clip}(g, C) = g \cdot \min(1, \frac{C}{\|g\|_2})$, which bounds the per-example gradient $\ell_2$-norm by a pre-defined threshold $C$, thus constraining the influence of any individual data point on the model update. However, the effectiveness of DP-SGD is critically dependent on the selection of $C$ \cite{andrew2021differentially, golatkar2022mixed,he2022exploring, pichapati2019adaclip}. An overly large $C$ decreases the effect of clipping and requires injecting more noise (proportional to $C$), potentially degrading utility. In contrast, an excessively small $C$ distorts gradient signals, hindering convergence.

To mitigate this sensitivity, adaptive clipping strategies have emerged. Auto-S \cite{bu2023automatic} and NSGD \cite{yang2022normalized} normalize gradients by implicitly absorbing $C$ into the effective learning rate of non-adaptive optimizers, driving gradients toward unit norm with a stabilizer $r$: $\text{Clip}(g,C=1) = \frac{g}{\| g \|_2 + r}$. DP-PSAC \cite{xia2023differentially} further introduces a non-monotonic clipping function: $\text{Clip}(g,C)=\frac{C}{\|g\|_2+r/\left(\|g\|_2+r\right)}$, which preserves the clipping bound ($\leq C$) required for privacy while reducing deviation from the true mini-batch-averaged gradients. By avoiding over-emphasis on small-norm gradients and limiting the effect of large-norm gradients, DP-PSAC effectively balances gradient distortion and noise magnitude.

\begin{figure*}[ht]
\centering
\includegraphics[width=1.00\textwidth]{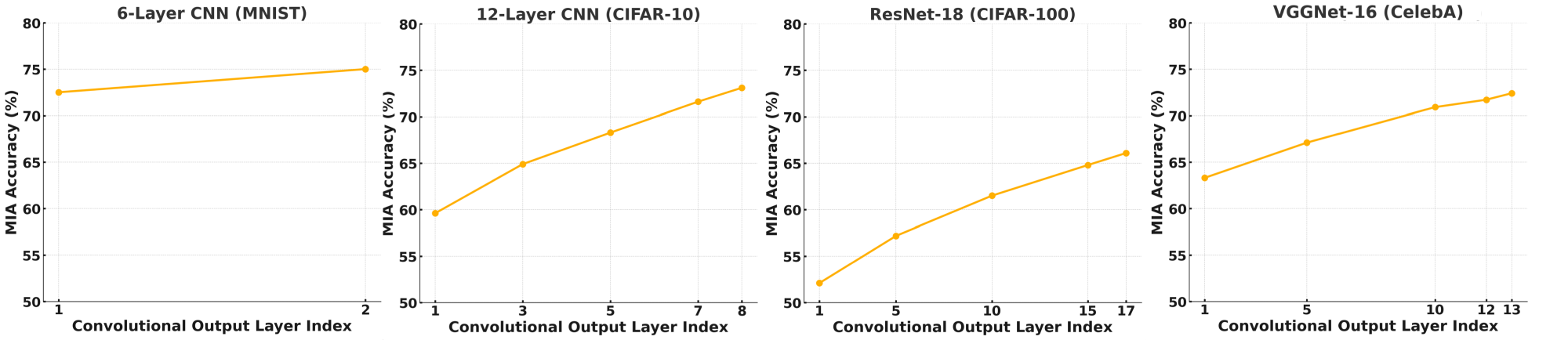}
\caption{MIA accuracy using intermediate representations from different convolutional layers.}
\label{evidence}
\end{figure*}

\subsection{DP-SGD Limitations for Layer-wise Privacy}
In EaaI settings, privacy of IRs is the concern. Yet standard DP-SGD and its adaptive clipping variants overlook heterogeneity in layer-wise MIA risk. By clipping the entire gradient to a bound $C$ and adding calibrated Gaussian noise, they neglect per-layer MIA-risk variations. Consequently, deeper layers—often more vulnerable to MIAs—receive insufficient protection, while shallower, less sensitive layers are over-perturbed, ultimately degrading the privacy–utility trade-off. 

An intuitive approach to implement layer-wise DP-SGD is to enforce DP at the granularity of individual layers, assigning each layer $l$ a distinct budget $(\varepsilon_l, \delta_l)$. However, by DP composition theorems \cite{kairouz2015composition}, the cumulative privacy cost $(\varepsilon, \delta)$ scales approximately linearly with model depth, i.e., $\varepsilon \approx \sum_l \varepsilon_l$. This rapid privacy budget accumulation makes naive per-layer privatization costly and impractical for deep models.

\section{Preliminaries}

\subsection{DP-SGD}\label{DP-SGD}

Standard Differentially Private Stochastic Gradient Descent (DP-SGD) enforces $(\varepsilon, \delta)$-DP by bounding each individual example’s contribution to the gradient update. At iteration $t\in\{0,...,T-1\}$, a mini-batch $\mathcal{B}_t$ is sampled from the private training set $D=\{(\mathbf{x}_n,\mathbf{y}_n)\}_{n=1}^N$ with probability $q \approx \frac{|\mathcal{B}_t|}{N}$. Subsequently, for each sample $\mathbf{x}_i \in \mathcal{B}_t$, its gradient $G_{t,i}=\nabla F(\mathbf{x}_i, W_t)$ is clipped as $\tilde{G}_{t,i} = G_{t,i} \cdot \frac{C_{t,i}}{\|G_{t,i}\|_2}$, where $C_{t,i} = \min(C, \|G_{t,i}\|_2)$ is the clipped $\ell_2$-norm, satisfying $C_{t,i}\le C$. The model parameters are then updated as $W_{t+1} \leftarrow W_t - \frac{\eta_t}{|\mathcal{B}_t|} \left( \sum_{\mathbf{x}_i \in \mathcal{B}_t} \tilde{G}_{t,i} + \mathcal{N}(0, C^2 \sigma^2 I) \right)$, where $\eta_t$ is the learning rate and $\sigma$ is the noise multiplier, calibrated via accounting methods such as the moments accountant \cite{wang2019subsampled} to ensure $(\varepsilon,\delta)$-DP.

\subsection{Layer-wise Heterogeneity of MIA Risk}\label{sec:evidence}
Intermediate Representations (IRs)—outputs of hidden layers in models—are vulnerable to Membership Inference Attacks (MIAs), which aim to determine whether a specific example was used for model training. Previous work \cite{shokri2017membership,nasr2019comprehensive,song2021systematic} indicates that this vulnerability is layer-dependent, with deeper IRs typically leaking more membership information. Our empirical results in Figure~\ref{evidence} further corroborate this trend: IRs from deeper layers consistently exhibit higher attack accuracy, indicating elevated MIA risk.

\section{Proposed Method: LM-DP-SGD}

This section introduces Layer-wise MIA-risk-aware DP-SGD (LM-DP-SGD), depicted in Figure \ref{fig:oveview}. Key symbols and definitions are summarized in Appendix \ref{appendix notation}.

\subsection{Layer-wise MIA-Risk Estimation}\label{4.1}

Prior to private training, we assess per-layer MIA risk through attack simulation. Given white-box access to the target model $F$, we (model trainers) instantiate a shadow model $F_{\text{shadow}}$ with an identical architecture. Building on the prior evidence that MIAs can transfer across datasets—i.e., adversaries trained on one dataset can generalize to another \cite{salem2018ml}—we train layer-specific MIA adversaries using a public shadow dataset $D_{\text{shadow}}=\{(\textbf{x}_n,\textbf{y}_n)\}_{n=1}^{N_{\text{s}}}$. The resulting MIA risk estimates derived from $D_{\text{shadow}}$ serve as reliable proxies for those on the private training set $D$, incurring no additional privacy cost. Specifically, we first partition $D_{\text{shadow}}$ into disjoint training and test subsets: $D_{\text{shadow}}^{\text{train}}=\{(\mathbf{x}_{n}^{'\text{in}'},\mathbf{y}_n^{'\text{in}'})\}_{n=1}^{N_{\text{in}}}$ and $D_{\text{shadow}}^{\text{test}}=\{(\mathbf{x}_{n}^{'\text{out}'},\mathbf{y}_n^{'\text{out}'})\}_{n=1}^{N_{\text{out}}}$, and train $F_{\text{shadow}}$ on $D_{\text{shadow}}^{\text{train}}$. Subsequently, we initialize per-layer IR sets for members and non-members, $\{I_{\text{in}}^{(l)}=\varnothing\}_{l=1}^{L},\{I_{\text{out}}^{(l)}=\varnothing\}_{l=1}^{L}$. All examples from $D_{\text{shadow}}^{\text{train}}$ (members) and $D_{\text{shadow}}^{\text{test}}$ (non-members) are then forwarded through $F_{\text{shadow}}$ to populate these sets, yielding $\{I_{\text{in/out}}^{(l)}\}_{l=1}^L$, where $I_{\text{in/out}}^{(l)}=\{I^{(l)}(\mathbf{x})\mid \mathbf{x} \in D_{\text{shadow}}^{\text{train/test}}\}$, and $I^{(l)}(\mathbf{x})$ denotes the IR of example $\mathbf{x}$ at layer $l$. For each layer $l$, we construct an adversary dataset $D_{\text{adv}}^{(l)}$ by pairing IRs in $I^{(l)}$ with their corresponding membership labels (1 for members, 0 for non-members), and use it to train a layer-specific adversary $F_{\text{adv}}^{(l)}(I^{(l)}(\mathbf{x});W^{(l)})$. Each adversary is then evaluated on its full training set to obtain an Error Rate (ER), producing $\{\text{ER}^{(l)}\}_{l=1}^L$. These per-layer ERs serve as MIA-risk estimates to guide layer-wise privacy budget allocation during target model training. The complete algorithmic procedure is detailed in Appendix~\ref{estimation}.

\subsection{Design of Layer-wise DP-SGD}\label{design}
To account for the heterogeneous vulnerability of IRs across layers to MIAs, we extend the standard DP-SGD (Section~\ref{DP-SGD}) to a layer-wise formulation. At iteration $t$, we introduce a reweighting vector $w_t = [w_t^{(1)}, \dots, w_t^{(L)}]$ constrained by $\sum_{l=1}^L (w_t^{(l)})^2 = 1$. Each component $w_t^{(l)}$ rescales the contribution of layer $l$'s $\ell_2$-norm relative to the overall gradient norm. Rather than clipping the entire gradient for each $\mathbf{x}_i\in \mathcal{B}_t$, we apply layer-wise reweighted clipping mechanism, yielding $\hat{G}_{t,i}=[\hat{g}_{t,i}^{(1)},...,\hat{g}_{t,i}^{(L)}]$. Concretely, each $G_{t,i}$ is decomposed into layer-wise components $G_{t,i}=[g_{t,i}^{(1)}, \dots, g_{t,i}^{(L)}]$. For each layer $l$, we define the normalized direction $d_{t,i}^{(l)} = \frac{g_{t,i}^{(l)}}{\|g_{t,i}^{(l)}\|_2}$ (or $0$ if $\|g_{t,i}^{(l)}\|_2 =0$) and compute the reweighted clipped gradient:
\begin{equation} 
    \hat{g}_{t,i}^{(l)} = C_{t,i} \cdot w_t^{(l)} \cdot d_{t,i}^{(l)} = C_{t,i} \cdot w_t^{(l)} \cdot \frac{g_{t,i}^{(l)}}{\|g_{t,i}^{(l)}\|_2},
\label{eq2}
\end{equation}
where $C_{t,i} = \min(C, \|G_{t,i}\|_2)$ enforces gradient clipping and $w_t$ reweights layer contributions. The squared $\ell_2$-norm of the layer-wise reweighted-and-clipped gradient is $\|\hat{G}_{t,i}\|_2^2 = \sum_{l=1}^L \|\hat{g}_{t,i}^{(l)}\|_2^2  = \sum_{l=1}^L C_{t,i}^2 (w_t^{(l)})^2 \|d_{t,i}^{(l)}\|_2^2$. Since each normalized direction satisfies $\|d_{t,i}^{(l)}\|_2^2 = 1$ (or $0$ when $\|g_{t,i}^{(l)}\|_2 =0$), this expression simplifies to $\|\hat{G}_{t,i}\|_2^2 = C_{t,i}^2 \sum_{l=1}^L (w_t^{(l)})^2$. Under the constraint $\sum_{l=1}^L (w_t^{(l)})^2 = 1$, the global $\ell_2$-norm of $\hat{G}_{t,i}$ is identical to that of the clipped gradient in standard DP-SGD, satisfying
\begin{equation}
\|\hat{G}_{t,i}\|_2 = C_{t,i} = \min(C, \|G_{t,i}\|_2) \le C.
\end{equation}
Thus, layer-wise reweighted clipping redistributes gradient magnitudes across layers while preserving the global $\ell_2$-norm bound $C$ required for privacy accounting. This enables layer-specific privacy allocation without increasing privacy budget. Accordingly, the model update is given by
\begin{equation}
    W_{t+1} \leftarrow W_t - \frac{\eta_t}{|\mathcal{B}_t|} \left( \sum_{\mathbf{x}_i \in \mathcal{B}_t} \hat{G}_{t,i} + \mathcal{N}(0, C^2 \sigma^2 I) \right).
     \label{eq:ladpsgd_update}
\end{equation}

\subsection{Private Training via LM-DP-SGD}\label{4.2}

Using the per-layer MIA risk estimates from Section \ref{4.1}, we train the target model $F$ using LM-DP-SGD. Within the layer-wise DP-SGD framework described in Section \ref{design}, at each iteration $t$, we optimize a layer-wise reweighting vector $w_t$ to balance two competing objectives: utility and privacy. On the utility side, our objective is to minimize the gradient bias induced by layer-wise reweighted clipping. Let $\bar{G}_t=\frac{1}{|\mathcal{B}_t|}\sum_{\mathbf{x}_i \in \mathcal{B}_t}\hat{G}_{t,i}$ denote the mini-batch-averaged layer-wise reweighted-and-clipped gradient, and let $\nabla F(W_t)=\mathbb{E}_{\textbf{x}_i\sim D}\left[\nabla F(\mathbf{x}_i;W)\right]$ be the expected true gradient. Their discrepancy can be decomposed into a layer-wise reweighted clipping bias ($b_t = \mathbb{E}[\hat{G}_{t,i}] - \nabla F(W_t)$ and sampling stochasticity ($ V_t=\bar{G}_t-\mathbb{E}[\hat{G}_{t,i}]$), such that $\bar{G}_t=\nabla F(W_t)+b_t+V_t$. In practice, $V_t$ is typically negligible, while $b_t$ dominates the bias and can distort parameter updates, slow convergence, and degrade final utility. To alleviate these effects, we optimize $w_t$ by solving $\min_{w_t}\|b_t\|_2^2$ s.t. $\sum_{l=1}^{L}(w_t^{(l)})^2=1$, which admits a closed-form solution via Lagrange multipliers (see Appendix~\ref{appendix lag}). On the privacy side, as shown in Section~\ref{sec:evidence}, the vulnerability of IRs to MIAs is layer-dependent and generally increases with network depth. Accordingly, we calibrate $w_t$ to assign stronger protection to more vulnerable (typically deeper) layers, yielding a MIA-risk-aware reweighting scheme that accounts for layer-wise membership exposure.
\begin{algorithm*}[ht]
\caption{Private Training via LM-DP-SGD.}
{\raggedright
\textbf{Inputs}: target model $F(\mathbf{x};W_0)$ with $L$ layers and randomly initialized $W_0$, training set $D$ of $N$ samples, estimated per-layer MIA risks (quantified by error rates) $\{\text{ER}^{(l)}\}_{l=1}^L$.\\
\textbf{Parameters}: sampling probability $q$, clipping threshold $C$, learning rate $\{\eta_t\}_{t=0}^{T-1}$, privacy budget $(\varepsilon, \delta)$, training iteration $T$, layer-wise MIA-risk emphasis factor $r$.\\
}
\vspace{0ex}
\begin{algorithmic}[1] 
\STATE Compute noise multiplier $\sigma$ s.t. $\varepsilon_{\text{Accountant}}(\delta, \sigma, q, T) \leq \varepsilon$ (e.g., via moments accountant).
\FOR{Iteration $t = 0$ to $T-1$}
    \STATE Sample a mini-batch: $\mathcal{B}_t \sim \text{Poisson}(q;D)$.\\
    \FOR{Each example $\mathbf{x}_i \in \mathcal{B}_t$}
    \STATE Compute the corresponding per-example gradient: $G_{t,i} = \nabla F(\mathbf{x}_i,W_t)=[g_{t,i}^{(1)}, ..., g_{t,i}^{(L)}]$.\\
    \STATE Compute the $\ell_2$-norm of the clipped gradient in standard DP-SGD: $C_{t,i} = \min(C, \|G_{t,i}\|_2)$.
    \ENDFOR
    \FOR{Layer $l=1$ to $L$}
        \STATE Calculate MIA-risk-aware reweighting coefficient for layer $l$: \\$\hat{w}_t^{(l)*} = \frac{1}{|\mathcal{B}_t|}\sum_{\mathbf{x}_i \in \mathcal{B}_t}\frac{\|g_{t,i}^{(l)}\|_2}{C_{t,i}}$, $\tilde{w}^{(l)*}_t = \hat{w}_{t}^{(l)}\cdot ( \text{ER}^{(l)})^r$, $w^{(l)}_t = \frac{\tilde{w}^{(l)*}_t}{\sqrt{\sum_{m=1}^L (\tilde{w}^{(m)*}_t)^2}}$.\\
    \ENDFOR
    \STATE Initialize per-example layer-wise reweighted-and-clipped gradients: $\{\hat{G}_{t,i} = \mathbf{0}\}_{\mathbf{x}_i \in \mathcal{B}_t}$.
    \FOR{Each example $\mathbf{x}_i \in \mathcal{B}_t$}
        \FOR{Layer \( l = 1 \) to \( L \)}
            \STATE Compute unit direction: $d_{t,i}^{(l)} = \frac{g_{t,i}^{(l)}}{\|g_{t,i}^{(l)}\|_2}$ (or $0$ if $\|g_{t,i}^{(l)}\|_2=0$).\\
            \STATE Compute layer-$l$ component of $\hat{G}_{t,i}$ and assign it accordingly: $\hat{g}_{t,i}^{(l)} = C_{t,i} \cdot w^{(l)}_t \cdot d_{t,i}^{(l)}$.
        \ENDFOR
    \ENDFOR
    \STATE Compute the differential private update gradient: $\hat{G}_t = \frac{1}{|\mathcal{B}_t|}\left[\sum_{\mathbf{x}_i \in \mathcal{B}_t} \hat{G}_{t,i}+\mathcal{N}(0, C^2 \sigma^2 I)\right]$.
    \STATE Update model parameters: $W_{t+1} = W_{t} - \eta_t \cdot \hat{G}_t$.
\ENDFOR
\STATE \textbf{return} Trained differentially private model $F(\mathbf{x};W_{T})$.
\end{algorithmic}
\label{alg:LR-DP-SGD}
\end{algorithm*}

To balance the two objectives, we assign the layer-wise reweighting coefficients through the following steps.
\begin{enumerate}
\item \textit{Initialization to eliminate layer-wise clipping bias}. For optimal convergence, we initialize $w_t$ to eliminate the bias between $\mathbb{E}[\hat{G}_{t,i}]$ and $\nabla F(W)$, enforcing the unbiasedness condition $\mathbb{E}[\hat{G}_{t,i}]-\nabla F(W)=\textbf{0}$. In practice, sampling-induced randomness is typically negligible; thus, to reduce computational overhead, we approximate $\mathbb{E}[\hat{G}_{t,i}] \simeq \bar{G}_t$ and $\nabla F(W) \simeq \mathbb{E}_{\mathbf{x}_i\sim \mathcal{B}_t}\left[\nabla F(\mathbf{x}_i; W)\right]$. From Eq.(\ref{eq2}), the layer-$l$ component of $\bar{G}_t$ is $w_t^{(l)} \frac{1}{|\mathcal{B}_t|}\sum_{\mathbf{x}_t\in \mathcal{B}_t}\left[C_{t,i}\frac{g_{t,i}^{(l)}}{\|g_{t,i}^{(l)}\|_2}\right]$, and the corresponding component of $\mathbb{E}_{\mathbf{x}_i\sim \mathcal{B}_t}\left[\nabla F(\mathbf{x}_i; W)\right]$ is $\mathbb{E}_{\mathbf{x}_i \sim \mathcal{B}_t}\left[g_{t,i}^{(l)}\right]$. Equating these two expressions yields the initialization: $\hat{w}_t^{(l)*} = \frac{1}{|\mathcal{B}_t|}\sum_{\mathbf{x}_i \in \mathcal{B}_t}\frac{\|g_{t,i}^{(l)}\|_2}{C_{t,i}}$.
\item \textit{Layer-wise MIA-risk calibration.} We then calibrate the initialized coefficients using per-layer MIA risk estimates from Section~\ref{4.1}. For each layer $l$, a lower $\text{ER}^{(l)}$ indicates that the layer retains stronger membership signals and is more susceptible to MIAs. To mitigate this vulnerability, we downweight its initialized coefficient $w_t^{(l)*}$ to enhance the relative protection afforded by the fixed-magnitude noise applied uniformly across layers. To control the emphasis on layer-wise MIA-risk heterogeneity, we introduce an emphasis factor $r\ge 1$, yielding adjusted risks $(\text{ER}^{(l)})^{r}$ and calibrated coefficients $\tilde{w}^{(l)*}_t = w^{(l)*}_t\cdot(\text{ER}^{(l)})^r$. A larger $r$ concentrates stronger protection on high-risk layers at the expense of less vulnerable ones, potentially degrading utility. In practice, $r$ is tuned to balance privacy and utility.
\item \textit{$\ell_2$-normalization.} Finally, the unit-norm constraint is enforced via normalization: $w^{(l)}_t = \frac{\tilde{w}^{(l)*}_t}{\sqrt{\sum_{m=1}^L (\tilde{w}^{(m)*}_t)^2}}$.
\end{enumerate}
With $w_t = [w^{(1)}_t,...,w^{(L)}_t]$, for each example $\mathbf{x}_i \in \mathcal{B}_t$ and layer $l = \{1,\dots,L\}$, we apply layer-wise reweighting and clipping to obtain $\hat{g}_{t,i}^{(l)} = C_{t,i} \cdot w^{(l)}_t \cdot d_{t,i}^{(l)}$, where $d_{t,i}^{(l)} = \frac{g_{t,i}^{(l)}}{\|g_{t,i}^{(l)}\|_2}$, and assign it to the layer-$l$ component of $\hat{G}_{t,i}$. The model parameters are then updated as $W_{t+1} = W_{t} - \eta_t \cdot \frac{1}{|\mathcal{B}_t|}\big[\sum_{\mathbf{x}_i \in \mathcal{B}_t} \hat{G}_{t,i}+\mathcal{N}(0, C^2 \sigma^2 I)\big]$. The complete training procedure via LM-DP-SGD is presented in Algorithm \ref{alg:LR-DP-SGD}.

\section{Theoretical Analysis}
\subsection{Privacy Guarantee}
LM-DP-SGD offers rigorous privacy guarantees, as formalized in Theorem~\ref{thm:privacy}, by bounding the $\ell_2$-sensitivity of the aggregated gradients before noise injection. The consistent $\ell_2$-norm of layer-wise reweighted-and-clipped gradients with standard DP-SGD allows for direct application of DP-SGD's established moments accountant framework. The proof of Theorem \ref{thm:privacy} is detailed in Appendix \ref{proof of privacy}.

\begin{theorem}[Privacy guarantee of LM-DP-SGD] \label{thm:privacy}
Algorithm~\ref{alg:LR-DP-SGD}, which applies a layer-wise reweighting vector $w_t=[w_t^{(1)},...,w_t^{(L)}]$ constrained by $\sum_{l=1}^L (w_t^{(l)})^2 = 1$ and adds noise $\mathcal{N}(0, C^2 \sigma^2 I)$ to the aggregated reweighted-and-clipped gradients $\sum_{\mathbf{x}_i \in \mathcal{B}_t} \hat{G}_{t,i}$ at each iteration $t$, guarantees $(\varepsilon, \delta)$-differential privacy under the same conditions as standard DP-SGD, with identical clipping threshold $C$, sampling probability $q$, and noise multiplier $\sigma$. In particular, by the moments accountant, there exist constants $c_1, c_2$ such that for any $\varepsilon \leq c_1 q^2 T$ and $\delta > 0$, choosing
\begin{align} \label{eq:sigma_bound}
    \sigma \geq c_2 \frac{q\sqrt{T\log(1/\delta)}}{\varepsilon}
\end{align}
ensures $(\varepsilon, \delta)$-differential privacy after $T$ iterations.
\end{theorem}

\subsection{Convergence Properties}

The convergence of LM-DP-SGD for non-convex objectives, a common assumption in deep learning, is established in Theorem~\ref{thm:convergence}, with the proof provided in Appendix~\ref{proof of convergence}.

\begin{theorem}[Convergence of LM-DP-SGD]\label{thm:convergence}
Let $F: \mathbb{R}^d \to \mathbb{R}$ be an $L$-smooth differentiable objective function with bounded expected true gradient, i.e., $\|\nabla F(W)\|_2 \le M$ for all $W$. Consider Algorithm~\ref{alg:LR-DP-SGD} executed for $T$ iterations from initialization $W_0$ on a dataset $D$ of size $N$, using a fixed learning rate $\eta_t = \eta/\sqrt{T}$ for some constant $\eta > 0$ (e.g., $\eta = 1/L$), Poisson sampling probability $q$, clipping threshold $C$, and a layer-wise reweighting vector $w_t=[w_t^{(1)},...,w_t^{(L)}]$ satisfying $\sum_l (w_t^{(l)})^2=1$. Define the mini-batch averaged, layer-wise reweighted-and-clipped gradient as $\bar{G}_t=\frac{1}{|\mathcal{B}_t|}\sum_{\mathbf{x}_i \in \mathcal{B}_t}\hat{G}_{t,i}$ with expectation $\mathbb{E}[\bar{G}_t]=\mathbb{E}[\hat{G}_{t,i}]$. Assume that the biases introduced by layer-wise reweighted clipping and mini-batch sampling are uniformly bounded, i.e., for all $t=\{0, \dots, T-1\}$, $\|\mathbb{E}[\bar{G}_t] - \nabla F(W_t)\|_2 = \|b_t\|_2 \le \xi$, and $\|\bar{G}_t-\mathbb{E}[\bar{G}_t]\|_2 \le V$. Let $\sigma$ denote the noise multiplier required for $(\varepsilon, \delta)$-differential privacy (cf. Eq.(\ref{eq:sigma_bound})). Then, the time-averaged squared $\ell_2$-norm of the expected true gradient admits:
\begin{align}
\begin{split}
    \frac{1}{T} \sum_{t=0}^{T-1}\|\nabla F(W_t)\|_2^2  \le O\left( \frac{L(F(W_0) - F^*)}{\sqrt{T}} \right) + O(M \xi) \\ +O\left( \frac{(M+\xi)^2 + V^2+ \frac{C^2\sigma^2 d}{q^2N^2}}{\sqrt{T}} \right),
\end{split}
\end{align}
where $d$ denotes the dimensionality of $W$. The expected objective function is defined as $F(W)=\mathbb{E}_{\mathbf{x}_i \sim D}[F(\mathbf{x}_i;W)]$ and $F^*$ is a lower bound on the attainable objective value under $W_T$ such that $F(W_T) \ge F^*$. The term $V^2$, induced by mini-batch sampling, is typically negligible in practice.
\end{theorem}

\textit{Analysis.} LM-DP-SGD achieves a convergence rate comparable to standard DP-SGD, converging to a neighborhood around a stationary point with radius $O(M\xi)$. Here, $M$ upper-bounds $\|\nabla F(W)\|_2$, determined by intrinsic factors such as the model architecture, loss function, training data, and applied constraints or regularization. We thus focus on $\xi = \sup_t \|b_t\|_2$, which captures the bias introduced by layer-wise reweighted clipping. Specifically, this bias quantifies the discrepancy between the expected layer-wise reweighted-and-clipped gradient $\mathbb{E}[\bar{G}_t]=\left[\mathbb{E}[\bar{G}_t^{(1)}],...,\mathbb{E}[\bar{G}_t^{(L)}]\right]$ and the expected true gradient $\nabla F(W_t)=[\nabla F^{(1)}(W_t),...,\nabla F^{(L)}(W_t)]$. For each layer $l$, $\mathbb{E}[\bar{G}_t^{(l)}]=\mathbb{E}_{\mathbf{x}_i\sim D}\left[ C_{t,i} w_t^{(l)} \frac{g_{t,i}^{(l)}}{\|g_{t,i}^{(l)}\|_2} \right] = w_t^{(l)} u_t$, where $u_t^{(l)} = \mathbb{E}_{\mathbf{x}_i\sim D}\left[ C_{t,i}\frac{g_{t,i}^{(l)}}{\|g_{t,i}^{(l)}\|_2} \right]$. The resulting bias is $b_t = [w_t^{(1)} u_t^{(1)} - \nabla F^{(1)}(W_t), \dots, w_t^{(L)} u_t^{(L)} - \nabla F^{(L)}(W_t)]$. In practice, as mini-batch sampling bias is typically negligible, expectations can be approximated by mini-batch averages to reduce computational cost: $\nabla F^{(l)}(W_t) \simeq \frac{1}{|\mathcal{B}_t|}\sum_{\mathbf{x}_i\in \mathcal{B}_t} g_{t,i}^{(l)}$, $u_t^{(l)}\simeq \frac{1}{|\mathcal{B}_t|}\sum_{\mathbf{x}_i\in \mathcal{B}_t}\left[C_{t,i}\frac{g_{t,i}^{(l)}}{\|g_{t,i}^{(l)}\|_2}\right]$.

Table~\ref{tab:bias_comparison} presents a conceptual comparison of the origins of the bias term $b_t$, whose $\ell_2$-norm is upper bounded by $\xi$, across LM-DP-SGD and other baselines. In deep learning, the non-convex and data-dependent nature of optimization, combined with the intricate interactions of privacy mechanisms, makes the theoretical characterization of $b_t$ challenging; consequently, prior work has largely relied on empirical analysis \cite{yang2022normalized, bu2023automatic, xia2023differentially}. DP-SGD \cite{abadi2016deep}, which employs a fixed clipping threshold $C$, introduces bias only through gradient clipping—often omitted or separately bounded—and empirically exhibits the smallest bias. Adaptive methods such as Auto-S/NSGD \cite{yang2022normalized, bu2023automatic} and DP-PSAC \cite{xia2023differentially} incur additional bias due to adaptive clipping or normalization, which is empirically larger than that of DP-SGD but remains of the same order. LM-DP-SGD introduces bias through layer-wise reweighted clipping, arising both from the reweighting of per-layer contributions to the global gradient norm (an additional source of bias) and from the clipping of these norms. We empirically evaluate and compare the magnitudes of these biases in Section~\ref{Magnitude of Gradient Bias}, and provide a detailed analysis of the layer-wise reweighted clipping bias in LM-DP-SGD in Appendix~\ref{bias}.
\begin{table}[ht]
\centering
\caption{Conceptual comparison of gradient norm constraint mechanisms across DP-SGD methods. $b_t$ captures the expected deviation between the constrained and true gradients.}
\resizebox{0.45\textwidth}{!}{
\begin{tabular}{cc}
\toprule
DP-SGD Method      & Gradient Norm Constraint Mechanism    \\ \midrule
DP-SGD      & Fixed clipping\\
Auto-S/NSGD & Adaptive clipping/Normalization \\
DP-PSAC     & Adaptive clipping/Normalization \\
LM-DP-SGD (Ours)  & Layer-wise reweighted clipping  \\ \bottomrule
\end{tabular}
}
\label{tab:bias_comparison}
\end{table}

\begin{table*}[ht]
\caption{Hyperparameter settings. Specifically, for MNIST, we follow the setup in \cite{shamsabadi2023losing}, employing a 6-layer Shallow-CNN (S-CNN). For CIFAR10, we adopt the configuration from \cite{papernot2021tempered}, using a 12-layer Deep-CNN (D-CNN).}
  \centering
  \resizebox{1\textwidth}{!}{
  \begin{tabular}{cccccccc}
    \toprule
    Training  & Target &Shadow &Privacy  & Sampling & Learning& Clipping & Layer-wise MIA-Risk  Heterogeneity \\ 
    Dataset $D$&Model $F$& Dataset $D_{\text{shadow}}$ &Budget $(\varepsilon,\delta)$&Probability $q$&Rate $\eta_t$&Threshold $C$ & Emphasis Factor $r$\\
    \midrule
    MNIST & S-CNN & FashionMNIST & $(5.0, 1\times10^{-5})$ &0.01&0.08 &1.0& 5.0\\
    CIFAR10 & D-CNN & ImageNet &$(8.0, 1\times10^{-5})$ &0.01&0.10    &2.0 &  3.0\\
    CIFAR100& ResNet-18  &  ImageNet  & $(8.0, 1\times10^{-5})$ &0.01&0.10 &3.0 & 2.0 \\
    CelebA & VGG-16 & FFHQ & $(8.0, 1\times10^{-5})$ &0.01&0.08 &3.0& 2.0\\
    \bottomrule
  \end{tabular}}
  \label{hyperparameter}
\end{table*}

\begin{table*}[ht]
\caption{Evaluation of MIA accuracy on partial intermediate convolutional-layer representations. The maximum MIA accuracy across layers indicates the maximum IR-level MIA risk, reflecting the overall vulnerability to MIAs at the IR level. \textbf{Bolded values} highlight the lowest peak MIA accuracy across layers among all methods, identifying the most effective method for mitigating IR-level MIA risk.}
    \centering
  \resizebox{1\textwidth}{!}{
    \begin{tabular}{cccccc}
    \toprule
      Model & DP-SGD  & Index of Intermediate & MIA Accuracy & MIA Accuracy  &Maximum MIA Accuracy across \\
      (Dataset) & Method & Convolutional-Layer & under SGD (\%) & under DP-SGD (\%)&  Layers under DP-SGD (\%) \\
    \midrule
         & DP-SGD & \multirow{4}{*}{1, 2}& \multirow{4}{*}{72.5, 75.0} & 67.2, 70.9 & 70.9\\
    S-CNN & Auto-S/NSGD &  &  &   66.6, 70.0&70.0       \\
(MNIST) & DP-PSAC  &  &  &   67.7, 71.8      &71.8\\
          & LM-DP-SGD (Ours)  & &  & 68.3, 69.2    &\textbf{69.2}  \\
    \hline
        & DP-SGD & \multirow{4}{*}{1, 3, 5, 7, 8}& \multirow{4}{*}{59.6, 64.9, 68.3, 71.6, 73.1} & 57.0, 63.2, 65.5, 67.3, 69.8 & 69.8\\
    D-CNN & Auto-S/NSGD &  &  &    58.2, 63.1, 66.4, 68.2, 70.7  &70.7  \\
(CIFAR10) & DP-PSAC  &  &  &    57.7, 62.8, 65.1, 68.0, 69.4  &69.4  \\
          & LM-DP-SGD (Ours)  & &  &   59.0, 64.3, 65.2, 66.5, 67.9& \textbf{67.9}   \\
    \hline
        & DP-SGD & \multirow{4}{*}{1, 5, 10, 15, 17}& \multirow{4}{*}{52.1, 57.2, 61.5, 64.8, 66.1} & 49.7, 54.5, 59.8, 62.7, 64.7&64.7\\
    ResNet-18 & Auto-S/NSGD &  &  &    49.9, 54.7, 59.0, 61.9, 63.9   &63.9 \\
(CIFAR100) & DP-PSAC  &  &  &   49.0, 53.8, 58.1, 61.0, 63.1  &63.1   \\
          & LM-DP-SGD (Ours)  & &  &  52.5, 55.9, 59.3, 60.2, 61.8 &\textbf{61.8}  \\\hline
               & DP-SGD & \multirow{4}{*}{1, 5, 10, 12, 13}& \multirow{4}{*}{63.3, 67.1, 70.9, 71.7, 72.4} & 57.1, 61.6, 65.2, 67.9, 68.5& 68.5 \\
    VGG-16 & Auto-S/NSGD &  &  &  57.3, 62.1, 66.4, 68.3, 69.9  &69.9    \\
(CelebA) & DP-PSAC  &  &  &   56.8, 61.2, 64.9, 66.7, 68.2  & 68.2   \\
          & LM-DP-SGD (Ours)  & &  &    60.6, 63.4, 64.1, 64.5, 65.6& \textbf{65.6} \\
    \bottomrule
    \end{tabular}}
    \label{noisy_baseline_table}
\end{table*}

\section{Experimental Setup}

We comprehensively evaluate LM-DP-SGD against standard DP-SGD \cite{abadi2016deep}, Auto-S/NSGD \cite{yang2022normalized, bu2023automatic}, and DP-PSAC \cite{xia2023differentially} on four real-world datasets: MNIST \cite{lecun1998gradient}, CIFAR10 \cite{krizhevsky2009learning}, CIFAR100\cite{krizhevsky2009learning}, and CelebA \cite{celeba}. Public shadow datasets approximate the private data distribution, with ablations in Appendix \ref{shadow_datset} demonstrating robustness to distributional mismatch. Hyperparameters are summarized in Table~\ref{hyperparameter}. All experiments are conducted on a server with two Intel(R) Xeon(R) Silver 4310 CPUs @2.10GHz and one NVIDIA A100 GPU, running Ubuntu 22.04 with CUDA Toolkit 12.4.

\section{Experimental Results}

Our evaluation focuses on: (i) resilience to MIAs \cite{nasr2019comprehensive} targeting IRs, quantified by the maximum MIA accuracy across all layers; (ii) model utility, measured by test accuracy; (iii) convergence behavior, characterized by the evolution of the $\ell_2$-norm of the gradient bias induced by layer-wise reweighted clipping (clipping in DP-SGD, adaptive clipping/normalization in Auto-S/NSGD and DP-PSAC) during training. Ablations are detailed in Appendix~\ref{ablation}.

\subsection{Enhanced Resilience against MIAs on IRs}\label{resilience}

This section demonstrates that LM-DP-SGD substantially improves resilience against MIAs targeting intermediate-layer representations. As reported in Table \ref{noisy_baseline_table}, LM-DP-SGD adaptively allocates privacy protection across layers, applying stronger safeguards to empirically more vulnerable deeper layers while incurring limited privacy loss in less vulnerable shallow layers. This layer-wise MIA-risk-aware protection allocation effectively reduces the maximum MIA accuracy observed across layers, achieving the lowest peak IR-level MIA risk among all baselines. The observed privacy gains arise from LM-DP-SGD’s layer-wise reweighted clipping: by down-weighting the norm contributions of empirically susceptible layers, the method amplifies the effectiveness of fixed-magnitude noise on those layers, thereby providing stronger and more targeted privacy protection.

\begin{figure*}[ht]
  \centering
  \subfigure[S-CNN (MNIST)]{
		\includegraphics[width=0.40\textwidth]{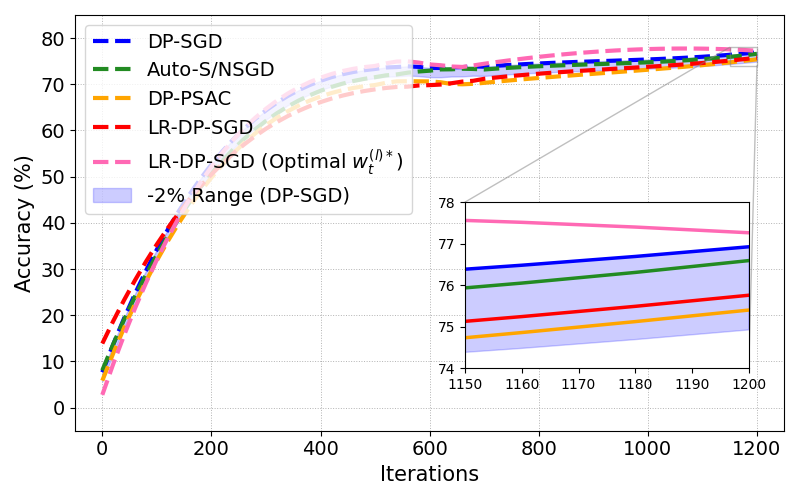}
    }
  \subfigure[D-CNN (CIFAR10)]{
		\includegraphics[width=0.40\textwidth]{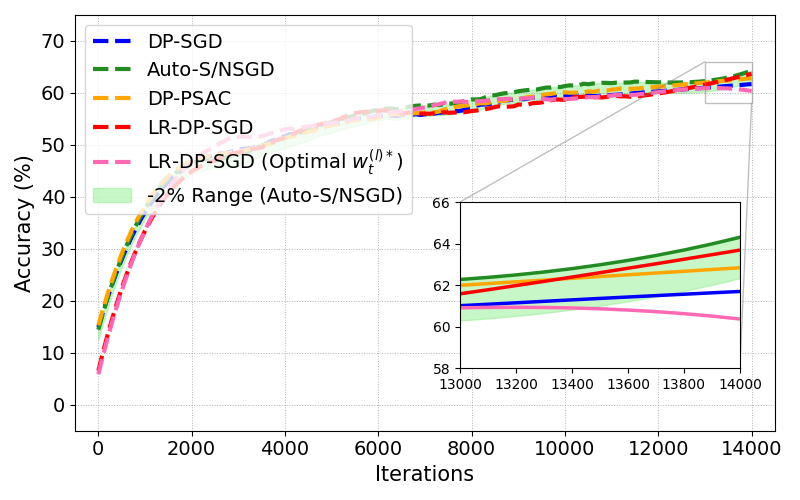}
		\label{acc_10}
    }
  \subfigure[ResNet-18 (CIFAR100)]{
		\includegraphics[width=0.40\textwidth]{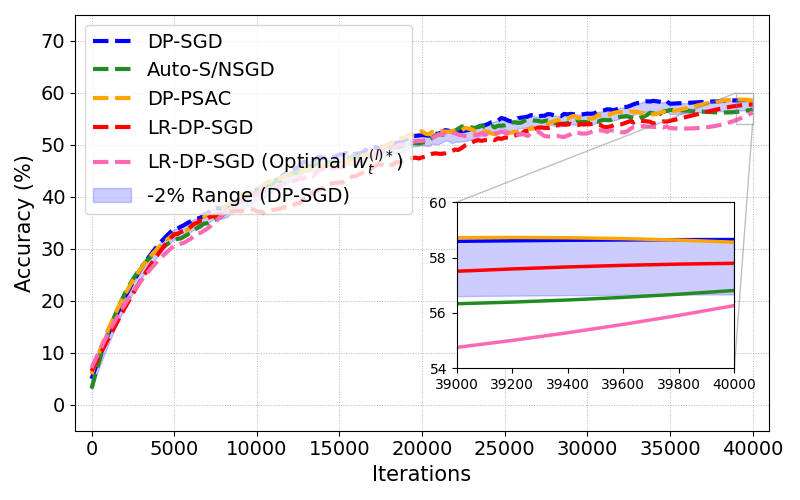}
		\label{acc_100}
    }
\subfigure[VGG-16 (CelebA)]{
		\includegraphics[width=0.40\textwidth]{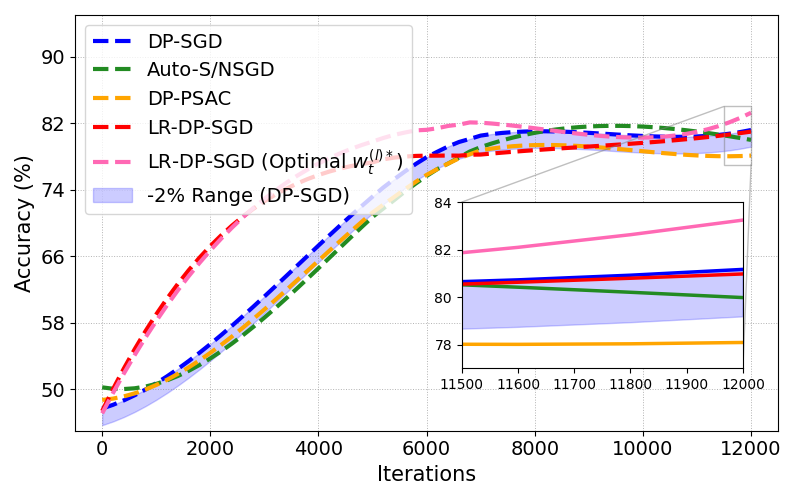}
		\label{acc_fmnist}
    }
  \caption{Evolution of test accuracy during training. The optimal layer-wise reweighting coefficients for convergence, $\{w_t^{(l)*}\}_{l=1}^L$, are derived via Lagrange multipliers (see Appendix \ref{appendix lag}). For clarity, all curves are smoothed using a Savitzky-Golay filter. The shaded regions denote performance within $2\%$ below the test accuracy of the baseline which achieves the maximum final accuracy. The results show that our method preserves model utility relative to the baselines. Importantly, this does not imply superior test accuracy; rather, performance remains comparable. This conclusion is supported by two observations. First, due to curve smoothing, minor deviations should be interpreted as smoothing artifacts rather than as meaningful performance differences; Second, once training converges, LM-DP-SGD’s performance consistently lies within the shaded regions. Overall, these results indicate that our method preserves utility.}
  \label{fig:test_accuracy_all_datasets}
\end{figure*}

\begin{figure*}[ht]
  \centering
  \subfigure[S-CNN (MNIST)]{
		\includegraphics[width=0.40\textwidth]{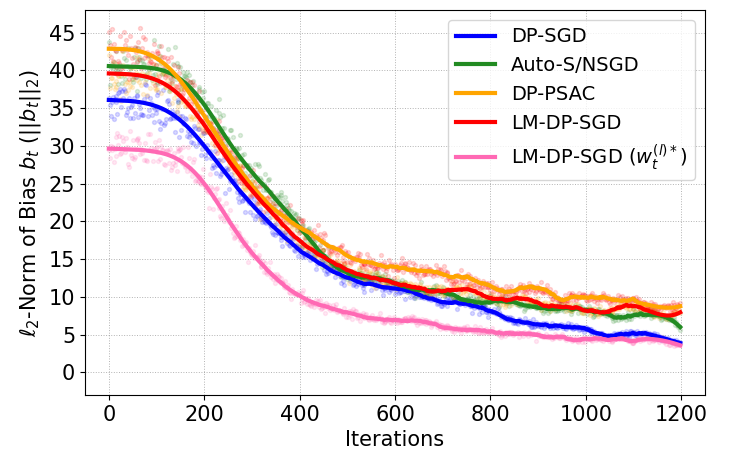}
    }
  \subfigure[D-CNN (CIFAR10)]{
		\includegraphics[width=0.40\textwidth]{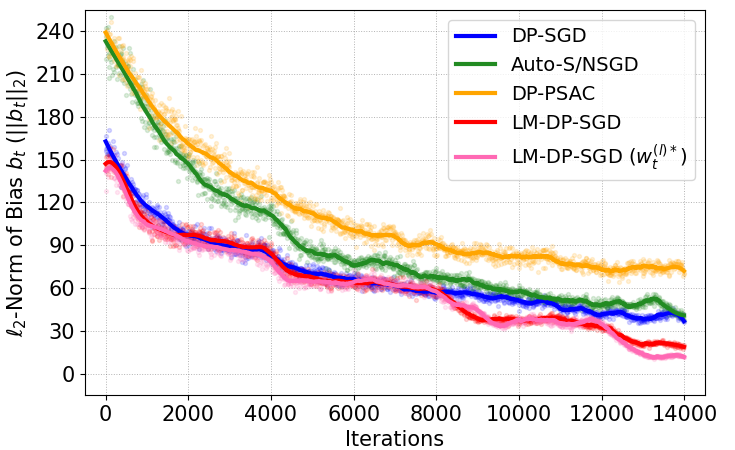}
    }
  \subfigure[ResNet-18 (CIFAR100)]{
		\includegraphics[width=0.40\textwidth]{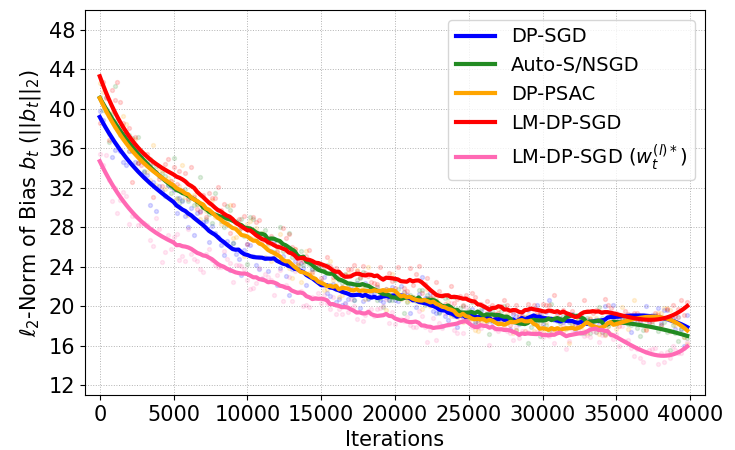}
    }
    \subfigure[VGG-16 (CelebA)]{
		\includegraphics[width=0.40\textwidth]{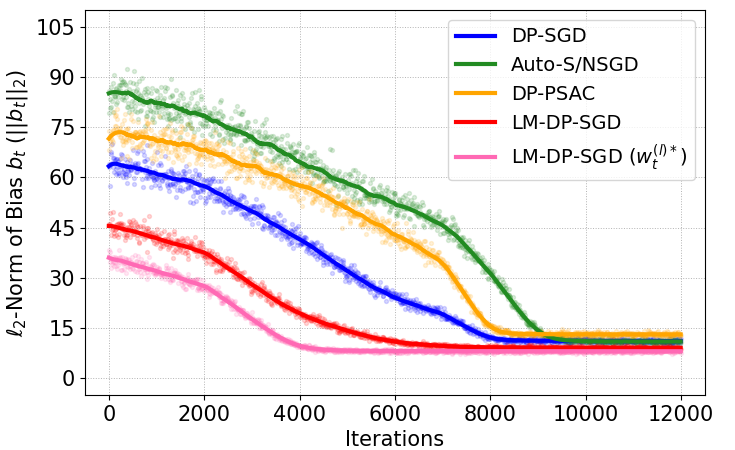}
    }
  \caption{Evolution of the $\ell_2$-norm of the bias term $b_t$, $\|b_t\|_2$. For visual clarity, all curves are smoothed using a Savitzky-Golay filter. We observe that the magnitude of $\|b_t\|_2$ remains within the same order across all methods throughout training. Compared to the baselines, our method does not introduce excessive bias; in fact, it yields a relatively lower gradient bias over the majority of training.}
  \label{fig:bias_norm_all_datasets}
\end{figure*}

Crucially, reducing the peak IR-level MIA risk directly limits the adversary’s maximal success probability of inferring membership from IRs without increasing the privacy budget. Since overall privacy leakage is ultimately dominated by the most vulnerable layer, suppressing the maximum layer-wise MIA accuracy prevents any single layer from becoming the bottleneck. Consequently, while shallow layers remain comparatively less sensitive, strengthened protection of deeper layers mitigates catastrophic worst-case leakage and improves end-to-end privacy robustness.

\subsection{Preserved Model Utility}\label{utility}

As depicted in Figure~\ref{fig:test_accuracy_all_datasets}: (i) LM-DP-SGD consistently achieves competitive—and in some cases superior—model utility, as measured by test accuracy, relative to all baseline methods. Since all curves are smoothed, minor deviations and fluctuations should be interpreted as the noise introduced by smoothing rather than as meaningful differences in model utility. To facilitate comparison, we highlight a shaded region corresponding to performance within 2\% below the maximum final test accuracy attained by any baseline, methods whose curves fall within this region are considered to have equivalent utility. We observe that LM-DP-SGD is consistently encompassed by the shaded region and is never the worst-performing method among the baselines. These results indicate that LM-DP-SGD effectively preserves model utility. When combined with the results in Section \ref{resilience}, which demonstrate LM-DP-SGD’s superior resilience against MIAs targeting IRs, these findings provide strong evidence that LM-DP-SGD achieves an improved privacy–utility trade-off; (ii) The final utilities attained by different DP-SGD approaches are highly comparable, suggesting that diverse gradient norm constraint mechanisms can effectively optimize model parameters and achieve similar levels of utility.

Furthermore, we analyze an LM-DP-SGD variant in which the layer-wise reweighting coefficients $\{w_t^{(l)*}\}_{l=1}^L$ are optimized solely for convergence via Lagrange multipliers (see Appendix~\ref{appendix lag}). Across all evaluated datasets, this variant converges faster and attains slightly higher final utility than the best-performing baselines. The reason is straightforward: baselines with global per-example clipping can be viewed as layer-wise DP-SGD without explicit layer-wise reweighting, where the corresponding coefficients are implicitly determined by each layer’s gradient norm. In contrast, explicit layer-wise reweighting in LM-DP-SGD ($\{w_t^{(l)*}\}_{l=1}^L$) aligns updates more closely with the true (unclipped) gradient, leading to more effective optimization, faster convergence, and often improved final performance.

\subsection{Training Convergence Behavior}\label{Magnitude of Gradient Bias}

To further analyze training convergence behavior, we examine the evolution of the $\ell_2$-norm of the bias term $b_t$, $\Vert b_t \Vert_2$, as illustrated in Figure~\ref{fig:bias_norm_all_datasets}. We highlight four key observations: (i) Across all evaluated datasets, the LM-DP-SGD variant with $\{w_t^{(l)*}\}_{l=1}^L$ induces the smallest bias, as its layer-wise reweighted and clipped gradients are more closely aligned with the true gradients; (ii) Among the baselines, standard DP-SGD exhibits the smallest bias, whereas adaptive clipping approaches (Auto-S/NSGD, DP-PSAC) generally incur higher bias, albeit of the same order of magnitude. Notably, despite being explicitly designed to balance heterogeneous layer-wise MIA risk with optimization dynamics, LM-DP-SGD achieves bias comparable to, or even smaller than, that of standard DP-SGD. On CelebA, in particular, LM-DP-SGD significantly outperforms standard DP-SGD in terms of gradient bias reduction; (iii) As training progresses, $\Vert b_t \Vert_2$ decreases monotonically, which is expected since convergence toward better optima yields smaller gradient norms and thus reduces the absolute distortion introduced by gradient norm constraints; (iv) Smaller $\Vert b_t \Vert_2$ indicates updates that are better aligned with the true gradients, leading to faster convergence. On CelebA, LM-DP-SGD and its variant with $\{w_t^{(l)*}\}_{l=1}^L$ achieve both the smallest bias and markedly faster convergence than competing baselines, consistent with the test accuracy trends in Figure~\ref{fig:test_accuracy_all_datasets}.

Overall, across different gradient norm constraint mechanisms, $\Vert b_t \Vert_2$ remains within the same order of magnitude. This indicates that the specific constraint design has only a marginal impact on training dynamics, underscoring the robustness of our layer-wise reweighted clipping strategy in meeting privacy guarantees and preserving model utility.

Additional details on running time and computational resources are provided in Appendix \ref{runningtime}.

\section{Conclusion}

This work investigates the empirically validated heterogeneity in vulnerability to Membership Inference Attacks (MIAs) across network layers, where Intermediate Representations (IRs) from deeper layers are typically more susceptible. To address the limitation of existing DP-SGD methods—which injects uniform noise into the globally clipped gradients and ignores layer-wise MIA-risk heterogeneity—we propose Layer-wise MIA-risk-aware DP-SGD (LM-DP-SGD). LM-DP-SGD adaptively calibrates the privacy protection strength of each layer based on its estimated MIA risk by performing layer-wise reweighted gradient clipping before noise injection. We also provide rigorous theoretical guarantees on both the privacy and convergence properties of LM-DP-SGD. Extensive experiments demonstrate that LM-DP-SGD markedly reduces peak MIA risks on IRs without increasing the overall privacy budget or degrading model utility, showcasing an improved privacy-utility trade-off.

\bibliography{example_paper}
\bibliographystyle{icml2026}

\newpage
\appendix
\onecolumn
\section{Notations}\label{appendix notation}
For clarity and consistency, Table \ref{tab:commands} provides a summary of key symbols and their corresponding descriptions.
\begin{table}[ht]
  \caption{Key symbols and descriptions used in LM-DP-SGD.}
  \label{tab:commands}
  \begin{tabular}{p{0.21\textwidth} p{0.75\textwidth}}
    \toprule
    Symbol & Description\\
    \midrule
    $F(\mathbf{x};W)$ & Target model with input $\mathbf{x}$ and parameters $W$. \\
    $F_{\text{shadow}}(\mathbf{x};W)$ & Shadow model with input $\mathbf{x}$ and parameters $W$.\\
    $I^{(l)}(\mathbf{x)}$& Intermediate representation of input $\mathbf{x}$ at layer $l$ of the trained shadow model.\\
    $F_{\text{adv}}^{(l)}(I^{(l)}(\mathbf{x});W^{(l)})$ & Adversarial model targeting layer-$l$ representations for membership inference attacks.\\
    $D$ & Private training set with $N$ examples.\\
    $D_{\text{shadow}}$ & Shadow dataset with $N_{\text{s}}$ examples.\\
    $\{\text{ER}^{(l)}\}_{l=1}^L$ & Per-layer MIA risks estimated by trained layer-specific adversaries, qualified by error rates.\\
    $G_{t,i}=[g_{t,i}^{(1)},...,g_{t,i}^{(L)}]$ & Gradient of the $i$-th example in $\mathcal{B}_t$: $G_{t,i}=\nabla F(\mathbf{x}_i,W_t)$, with $g_{t,i}^{(l)}$ the layer-$l$ component.\\
    $C_{t,i}=\min(C, \|G_{t,i}\|_2)$& $\ell_2$-norm of the clipped gradient of example $\mathbf{x}_i$ at iteration $t$ of standard DP-SGD.\\
    $\tilde{G}=[\tilde{g}^{(1)},...,\tilde{g}^{(L)}]$ & Resulting gradient obtained by applying clipping to $G$ in standard DP-SGD, with $\tilde{g}^{(l)}$ denoting the layer-$l$ component.\\
    $w_t=[w_t^{(1)},...,w_t^{(L)}]$ & Layer-wise reweighting vector at iteration $t$ in LM-DP-SGD, with $w_t^{(l)}$ for layer $l$.\\
    $\hat{G}=[\hat{g}^{(1)},...,\hat{g}^{(L)}]$ & Resulting gradient obtained by applying layer-wise reweighted clipping to $G$ in LM-DP-SGD, with $\hat{g}^{(l)}$ denoting the layer-$l$ component. \\
    $\bar{G}_t=\frac{1}{|\mathcal{B}_t|}\sum_{\mathbf{x}_i\in \mathcal{B}_t}[\hat{G}_{t,i}]$& Mini-batch-averaged layer-wise reweighted-and-clipped gradient at iteration $t$.\\
    $\nabla F(W_t)=\mathbb{E}_{\mathbf{x}_i \sim D}[G_{t,i}]$ & Expected true gradient at iteration $t$, which can be approximated by the mini-batch average in practice: $\frac{1}{|\mathcal{B}_t|} \sum_{\mathbf{x}_i \in \mathcal{B}_t} G_{t,i}$.\\
    \bottomrule
  \end{tabular}
\end{table}

\section{Layer-wise MIA-Risk Estimation Algorithm}\label{estimation}
The layer-wise estimation of MIA risk is conducted via the algorithmic procedure outlined in Algorithm~\ref{alg:estimation}, which is performed prior to differentially private training.

\begin{algorithm}[!ht]
\caption{Layer-wise MIA-Risk Estimation.}
{\raggedright
\textbf{Inputs}: shadow model $F_{\text{shadow}}(\mathbf{x};W)$ with $L$ layers and randomly initialized $W$, shadow dataset $D_{\text{shadow}}$, layer-specific MIA adversaries $\{F_{\text{adv}}^{(l)}(I^{(l)}(\mathbf{x}),W^{(l)})\}_{l=1}^L$ with randomly initialized $\{W^{(l)}\}_{l=1}^L$.\\
\textbf{Parameters}: \textit{During shadow model training—}train-test split ratio $r$, training epoch $E_{\text{shadow}}$, batch size $B_{\text{shadow}}$, learning rate $\eta_{\text{shadow}}$, and loss function $\ell_\text{s}$; \textit{During adversaries training—}training epochs $\{E_{\text{adv}}^{(l)}\}_{l=1}^L$, batch sizes $\{B_{\text{adv}}^{(l)}\}_{l=1}^L$, learning rates $\{\eta_{\text{adv}}^{(l)}\}_{l=1}^L$, and loss function $\ell_\text{a}$.\\
}
\vspace{0ex}
\begin{algorithmic}[1] 
\STATE Split $D_{\text{shadow}}$ into disjoint $D_{\text{shadow}}^{\text{train}}=\{(\mathbf{x}_{n}^{'\text{in}'},\mathbf{y}_n^{'\text{in}'})\}_{n=1}^{N_{\text{in}}}$ and $D_{\text{shadow}}^{\text{test}}=\{(\mathbf{x}_{n}^{'\text{out}'},\mathbf{y}_n^{'\text{out}'})\}_{n=1}^{N_{\text{out}}}$ with $N_{\text{in}} = \lfloor r \cdot N_{\text{s}} \rfloor$.\\
\FOR{Epoch $e=0$ to $E_{\text{shadow}}-1$}
    \FOR{Each mini-batch $\mathcal{B} \subset D_{\text{shadow}}^{\text{train}}$ with $|\mathcal{B}|=B_{\text{shadow}}$}
        \STATE Calculate $\hat{\mathbf{y}}_j = F_{\text{shadow}}(\mathbf{x}_j; W)$ for $(\mathbf{x}_j,\mathbf{y}_j)\in\mathcal{B}$.\\
        \STATE Calculate training loss: $\mathcal{L} \leftarrow \frac{1}{|\mathcal{B}|}\sum_{(\mathbf{x}_j,\mathbf{y}_j)\in\mathcal{B}} \ell_\text{s}(\hat{\mathbf{y}}_j, \mathbf{y}_j)$.\\
        \STATE Update shadow model: $W \leftarrow W-\eta_{\text{shadow}}\cdot\nabla_{W} \mathcal{L}$.\\
    \ENDFOR
\ENDFOR
\STATE Initialize empty per-layer IR sets: $\{I_{\text{in}}^{(l)}=\varnothing\}_{l=1}^{L},\{I_{\text{out}}^{(l)}=\varnothing\}_{l=1}^{L}$.\\
\FOR{Each example $(\mathbf{x},\mathbf{y}) \in D_{\text{shadow}}^{\text{train}}$}
    \STATE Forward $\mathbf{x}$ through $F_{\text{shadow}}$ and append the corresponding IR, $I^{(l)}(\mathbf{x})$, to $I_{\text{in}}^{(l)}$ for each layer $l=\{1,...,L\}$.\\
\ENDFOR
\FOR{Each example $(\mathbf{x},\mathbf{y}) \in D_{\text{shadow}}^{\text{test}}$}
    \STATE Forward $\mathbf{x}$ through $F_{\text{shadow}}$ and append the corresponding IR, $I^{(l)}(\mathbf{x})$, to $I_{\text{out}}^{(l)}$ for each layer $l=\{1,...,L\}$..\\
\ENDFOR
\FOR{Layer $l=1$ to $L$}
    \STATE Initialize training set of $F_{\text{adv}}^{(l)}$: $D_{\text{adv}}^{(l)}= \varnothing$.
    \FOR{Each IR $I^{(l)}(\mathbf{x})$ in $I_{\text{in}}^{(l)}$}
        \STATE Append $(I^{(l)}(\mathbf{x}), \mathbf{z}=1)$ to $D_{\text{adv}}^{(l)}$.\\
    \ENDFOR
    \FOR{Each IR $I^{(l)}(\mathbf{x})$ in $I_{\text{out}}^{(l)}$}
        \STATE Append $(I^{(l)}(\mathbf{x}), \mathbf{z}=0)$ to $D_{\text{adv}}^{(l)}$.\\
    \ENDFOR
    \FOR{Epoch $e=0$ to $E_{\text{adv}}^{(l)}-1$}
        \FOR{Each mini-batch $\mathcal{B} \subset D_{\text{adv}}^{(l)}$ with $|\mathcal{B}|=B_{\text{adv}}^{(l)}$}
            \STATE Calculate $\hat{\mathbf{z}}_j = F_{\text{adv}}^{(l)}(I_j^{(l)}(\mathbf{x}); W^{(l)})$ for $(I_j^{(l)}(\mathbf{x}), \mathbf{z}_j)\in \mathcal{B}$.\\
            \STATE Calculate training loss: $\mathcal{L} \leftarrow \frac{1}{|\mathcal{B}|}\sum_{(\mathbf{x}_j,\mathbf{z}_j)\in\mathcal{B}} \ell_\text{a}(\hat{\mathbf{z}}_j, \mathbf{z}_j)$.\\
            \STATE Update adversary: 
            $W^{(l)} \leftarrow W^{(l)} - \eta_{\text{adv}}^{(l)} \cdot \nabla_{W^{(l)}} \mathcal{L}$.
        \ENDFOR
    \ENDFOR
    \STATE Evaluate $F_{\text{adv}}^{(l)}$ on full $D_{\text{adv}}^{(l)}$ to obtain error rate $\text{ER}^{(l)}$.
\ENDFOR
\STATE \textbf{return} Estimated per-layer MIA risks $\{\text{ER}^{(l)}\}_{l=1}^L$.\\
\end{algorithmic}
\label{alg:estimation}
\end{algorithm}

\section{Theoretical Proofs}

\subsection{Proof of Theorem 5.1}\label{proof of privacy}
\begin{proof}
The core requirement for applying the Gaussian mechanism in DP-SGD is to bound the $\ell_2$-sensitivity of the function evaluated on the dataset-in this case, the sum of layer-wise reweighted-and-clipped gradients $\sum_{i \in \mathcal{B}_t} \hat{G}_{t,i}$. The sensitivity quantifies the maximum change in this sum induced by modifying a single data point in $D$. Under Poisson sampling, this is bounded by the maximum $\ell_2$-norm of the individual contributions \cite{abadi2016deep}. Each per-sample gradient $G_{t,i} = [g_{t,i}^{(1)}, \ldots, g_{t,i}^{(L)}]$ is layer-wise reweighted-and-clipped as $\hat{g}_{t,i}^{(l)}=C_{t,i}\cdot w_t^{(l)}\cdot \frac{g_{t,i}^{(l)}}{\|g_{t,i}^{(l)}\|_2}$, where the layer-wise reweighting coefficients $w_t^{(l)}$ satisfy $\sum_{l=1}^L (w_t^{(l)})^2 = 1$. Thus, $\|\hat{G}_{t,i}\|_2 =C_{t,i}\leq C$, ensuring that the $\ell_2$-sensitivity of the sum is bounded by $C$. Crucially, the coefficients $w_t^{(l)}$ are derived solely from the gradients themselves and from estimates of each layer’s susceptibility to MIAs, obtained using a public shadow dataset. Their computation therefore incurs no additional privacy cost. Consequently, adding Gaussian noise $\mathcal{N}(0, C^2 \sigma^2 I)$ calibrated to sensitivity $C$ guarantees the pre-defined $(\varepsilon,\delta)$-differential privacy, and the standard moments accountant analysis for DP-SGD applies without modification.
\end{proof}

\subsection{Proof of Theorem 5.2}\label{proof of convergence}
\subsubsection{Assumptions}

Theorem 5.2 sketches the convergence analysis of LM-DP-SGD under standard non-convex settings. To this end, we begin by making the following standard assumptions.
\begin{assumption}[Smoothness] \label{assump:smooth}
The objective function $F:\mathbb{R}^d\to\mathbb{R}$ is differentiable and $L$-smooth; that is, for all $W,W'\in \mathbb{R}^d$,
{\[
F(W') \leq F(W) + \langle \nabla F(W), W' - W \rangle + \frac{L}{2} \|W' - W\|_2^2.
\]}
\end{assumption}
\begin{assumption}[Bounded Bias and Variance] \label{assump:bias_variance}
Let $\nabla F(W_t) = \mathbb{E}_{\mathbf{x}_i \sim D}[\nabla F(\mathbf{x}_i; W_t)]$ denote the expected true gradient of the objective at iteration $t$ under the data distribution $D$. We assume the expected mini-batch-averaged layer-wise reweighted-and-clipped gradient $\mathbb{E}[\bar{G}_t]$ incurs a bounded bias $b_t$ relative to $\nabla F(W_t)$:
\[
\mathbb{E}[\bar{G}_t] = \mathbb{E}_{\mathbf{x}_i\sim D}[\hat{G}_{t,i}]=\nabla F(W_t) + b_t , \quad \text{with} \quad\|b_t\|_2 \le \xi,
\]
where $b_t$ captures the bias induced by layer-wise reweighted clipping, upper bounded by $\xi$. In addition, stochastic mini-batch sampling contributes an additional bounded sampling error $V_t$, so that
\[
\bar{G}_t = \mathbb{E}_{\mathbf{x}_i\sim D}[\hat{G}_{t,i}]+V_t=\mathbb{E}[\bar{G}_t]+V_t , \quad \text{with} \quad\|V_t\|_2 \le V,
\]
where $V$ bounds the sampling deviation and is typically negligible in practice. Thus, the mini-batch-averaged layer-wise reweighted-and-clipped gradient decomposes as $\bar{G}_t=\nabla F(W_t) + b_t+V_t$.
Moreover, for the differentially private update, let $Z_t = \mathcal{N}(0, C^2 \sigma^2 I)$ denote the added Gaussian noise, then the update gradient is $\hat{G}_t = \bar{G}_t + \frac{1}{|\mathcal{B}_t|} Z_t$. We bound the variance of $\hat{G}_t$ by separating the contribution of sampling variability and the explicit Gaussian perturbation:
\begin{align*}
    \mathbb{E}\|\hat{G}_t - \mathbb{E}[\hat{G}_t]\|_2^2 &=\mathbb{E}\|\bar{G}_t + \frac{1}{|\mathcal{B}_t|} Z_t - \mathbb{E}[\bar{G}_t]\|_2^2 \\&\le \mathbb{E}\|\bar{G}_t - \mathbb{E}[\bar{G}_t]\|_2^2 + \mathbb{E}\|\tfrac{1}{|\mathcal{B}_t|} Z_t\|_2^2=\mathbb{E}\|V_t\|_2^2 + \frac{C^2 \sigma^2 d}{|\mathcal{B}_t|^2} \\&\approx \mathbb{E}\|V_t\|_2^2 + \frac{C^2\sigma^2d}{q^2N^2} \le V^2+\frac{C^2\sigma^2d}{q^2N^2},
\end{align*}
where $d$ is the dimensionality of the parameter space $W$. Since $\mathcal{B}_t$ is Poisson-sampled from $D$, $|\mathcal{B}_t| \approx qN$ holds.
\end{assumption}

\begin{assumption}[Bounded Gradients] \label{assump:bounded_grad}
We assume that the $\ell_2$-norm of $\nabla F(W_t)$ is bounded, $\|\nabla F(W_t)\|_2 \le M$. And $\bar{G}_t$ admits $\|\bar{G}_t\|_2 \le C$, which follows from $\|\bar{G}_t\|_2 = \left\lVert \frac{1}{|\mathcal{B}_t|}\sum_{\mathbf{x}_i \in \mathcal{B}_t} \hat{G}_{t,i}\right\lVert_2 \le \frac{1}{|\mathcal{B}_t|}\sum_{\mathbf{x}_i \in \mathcal{B}_t} \|\hat{G}_{t,i}\|_2 = \frac{1}{|\mathcal{B}_t|}\sum_{\mathbf{x}_i \in \mathcal{B}_t} C_{t,i} \le C$.
\end{assumption}
\subsubsection{Detailed Proof}
\begin{proof}
    \textbf{One-Step Progress}. Leveraging the $L$-smoothness property (Assumption~\ref{assump:smooth}), the following inequality holds:
    \begin{align*}
            F(W_{t+1}) &\le F(W_t) + \langle \nabla F(W_t), W_{t+1} - W_t \rangle + \frac{L}{2} \|W_{t+1} - W_t\|_2^2 \\
    &= F(W_t) - \eta_t \langle \nabla F(W_t), \hat{G}_t \rangle + \frac{L \eta_t^2}{2}\|\hat{G}_t\|_2^2.
    \end{align*}
    Conditioning on $W_t$ and taking expectation over the stochasticity in $\hat{G}_t$ (from mini-batch sampling and noise perturbation),
    \begin{align*}
        \mathbb{E}[F(W_{t+1}) | W_t] \le F(W_t) - \eta_t \langle \nabla F(W_t), \mathbb{E}[\hat{G}_t | W_t] \rangle + \frac{L \eta_t^2}{2} \mathbb{E}[\|\hat{G}_t\|_2^2 | W_t].
    \end{align*}
    From Assumption~\ref{assump:bias_variance}, we know that $\mathbb{E}[\hat{G}_t | W_t] = \mathbb{E}[\bar{G}_t | W_t] = \nabla F(W_t) + b_t $. Substituting this into the inequality:
    \begin{align*}
        \mathbb{E}[F(W_{t+1}) | W_t] \le F(W_t) - \eta_t \langle \nabla F(W_t), \nabla F(W_t) + b_t \rangle + \frac{L \eta_t^2}{2} \mathbb{E}[\|\hat{G}_t\|_2^2 | W_t].
    \end{align*}
    The second moment of the stochastic gradient estimator, $\mathbb{E}[\|\hat{G}_t\|_2^2]$, can be decomposed into the squared norm of its mean and its variance: $\mathbb{E}[\|\hat{G}_t\|_2^2] = \|\mathbb{E}[\hat{G}_t]\|_2^2 + \text{Var}(\hat{G}_t)$, where $\|\mathbb{E}[\hat{G}_t]\|_2^2=\|\mathbb{E}[\bar{G}_t]\|_2^2=\|\nabla F(W_t)+b_t\|_2^2$ and $\text{Var}(\hat{G}_t) = \mathbb{E}\|\hat{G}_t - \mathbb{E}[\hat{G}_t]\|_2^2 = \mathbb{E}\|\hat{G}_t-\mathbb{E}[\bar{G}_t]|_2^2\le V^2 + \frac{C^2 \sigma^2 d}{|\mathcal{B}_t|^2}$ (Assumption~\ref{assump:bias_variance}). Combining these elements, we obtain the overall bound: $\mathbb{E}[\|\hat{G}_t\|_2^2] \le \|\nabla F(W_t) + b_t\|_2^2 + V^2 + \frac{C^2 \sigma^2 d}{|\mathcal{B}_t|^2}$. Substituting back into the inequality for $\mathbb{E}[F(W_{t+1}) | W_t]$:
    \begin{align*}
    \mathbb{E}[F(W_{t+1}) | W_t] \le F(W_t) - \eta_t \|\nabla F(W_t)\|_2^2 - \eta_t \langle \nabla F(W_t), b_t \rangle +  \frac{L \eta_t^2}{2} \left( \|\nabla F(W_t) + b_t\|_2^2 + V^2 + \frac{C^2 \sigma^2 d}{|\mathcal{B}_t|^2} \right).
    \end{align*}
    As $\langle \nabla F(W_t), b_t \rangle \ge -\|\nabla F(W_t)\|_2 \|b_t\|_2 \ge -M \xi$ and $\|\nabla F(W_t) + b_t\|_2^2 = \|\nabla F(W_t)\|_2^2 + \|b_t\|_2^2 + 2\langle \nabla F(W_t), b_t \rangle \le M^2 + \xi^2 + 2M\xi$,
    \begin{align*}
    \mathbb{E}[F(W_{t+1}) | W_t] \le F(W_t) - \eta_t \|\nabla F(W_t)\|_2^2 + \eta_t M \xi +\frac{L \eta_t^2}{2} \left( M^2 + \xi^2 + 2M\xi +V^2 + \frac{C^2 \sigma^2 d}{q^2 N^2} \right),
    \end{align*}
    where $|\mathcal{B}_t| \approx qN$. Rearranging terms to isolate gradient norm:
    \begin{align*}
    \eta_t \|\nabla F(W_t)\|_2^2 \le F(W_t) - \mathbb{E}[F(W_{t+1}) | W_t] + \eta_t M \xi +  \frac{L \eta_t^2}{2} \left( (M+\xi)^2 + V^2 + \frac{C^2 \sigma^2 d}{q^2 N^2} \right)
    \end{align*}
    \textbf{Telescoping and Convergence Rate.} Taking the total expectation and summing the inequality over $t=\{0, \dots, T-1\}$ ,and for simplicity, assuming a constant learning rate $\eta_t$:
    \begin{align*}
    &\quad \sum_{t=0}^{T-1} \eta_t \|\nabla F(W_t)\|_2^2 \\&\le \sum_{t=0}^{T-1} (F(W_t) - F(W_{t+1})) + \sum_{t=0}^{T-1} \eta_t M \xi + \sum_{t=0}^{T-1} \frac{L \eta_t^2}{2} \left( (M+\xi)^2 + V^2 + \frac{C^2 \sigma^2 d}{q^2 N^2} \right) \\
    &= F(W_0) - F(W_T) + T \eta_t M \xi + T \frac{L \eta_t^2}{2} \left( (M+\xi)^2 + V^2 + \frac{C^2 \sigma^2 d}{q^2 N^2} \right).
    \end{align*}
    Let $F^*$ be the optimal function value such that $ F(W_T)\ge F^*$, then,
    \begin{align*}
        \eta_t \sum_{t=0}^{T-1} \|\nabla F(W_t)\|_2^2 \le F(W_0) - F^* + T \eta_t M \xi + T \frac{L \eta_t^2}{2} \left( (M+\xi)^2 + V^2 + \frac{C^2 \sigma^2 d}{q^2 N^2} \right).
    \end{align*}
    Dividing by $T\eta_t$:
    \begin{align*}
        \frac{1}{T} \sum_{t=0}^{T-1} \|\nabla F(W_t)\|_2^2 \le \frac{F(W_0) - F^*}{T \eta_t} + M \xi + \frac{L \eta_t}{2} \left( (M+\xi)^2 + V^2 + \frac{C^2 \sigma^2 d}{q^2 N^2} \right).
    \end{align*}
    Setting the learning rate $\eta_t = \eta / \sqrt{T}$, and choosing $\eta = \Theta(1/L)$ (i.e., $\eta \propto 1/L$), the convergence guarantee is:
    \begin{align*}
        \frac{1}{T} \sum_{t=0}^{T-1} \|\nabla F(W_t)\|_2^2\le O\left( \frac{L(F(W_0) - F^*)}{\sqrt{T}} \right) + O(M \xi) + O\left( \frac{(M+\xi)^2 + V^2 + \frac{C^2\sigma^2 d}{q^2N^2}}{\sqrt{T}} \right).
    \end{align*}
    This result aligns with the convergence guarantee presented in Theorem 5.2. For sufficiently large $T$, the terms scaling with $\frac{1}{\sqrt{T}}$ diminish, leaving the bias term $O(M\xi)$ as a potential bottleneck for convergence to a true zero gradient.
\end{proof}

\section{Detailed Analysis of Bias Term $b_t$}\label{bias}
In Stochastic Gradient Descent (SGD), the gradient for $\mathbf{x}_i$ in $D$ at iteration $t$ is denoted by $G_{t,i} = [g_{t,i}^{(1)}, \dots, g_{t,i}^{(L)}]$, where $L$ is the total number of layers and $g_{t,i}^{(l)}$ denotes the gradient with respect to the parameters of the $l$-th layer. LM-DP-SGD extends this formulation by introducing the expected layer-wise reweighted-and-clipped gradient for each layer $l$:
\begin{equation} \label{eq:appendix_avg_weighted_grad}
\mathbb{E}[\hat{g}_{t,i}^{(l)}] = \frac{1}{N}\sum_{\mathbf{x}_i\in D} C_{t,i} \frac{w_t^{(l)}g_{t,i}^{(l)}}{\|g_{t,i}^{(l)}\|_2},
\end{equation}
where $C_{t,i} = \min(C, \|G_{t,i}\|_2)$ is the $\ell_2$-norm of the clipped gradient of $\mathbf{x}_i$ in standard DP-SGD, and $w^{(l)}_t$ is the reweighting coefficient for layer $l$ at iteration $t$. Then $\mathbb{E}[\hat{G}_{t,i}] = \left[\mathbb{E}[\hat{g}_{t,i}^{(1)}], \dots, \mathbb{E}[\hat{g}_{t,i}^{(L)}]\right]$, with the layer-wise reweighting coefficients constrained by $\sum_{l=1}^L (w^{(l)}_t)^2 = 1$.

The convergence analysis relies on the assumption that $\mathbb{E}[\hat{G}_{t,i}]= \nabla F(W_t) + b_t$, where $\nabla F(W_t)=\mathbb{E}_{\textbf{x}_i \sim D}[\nabla F(\textbf{x}_i; W_t)]$ is the expected true gradient, $b_t$ is the bias induced by layer-wise reweighted clipping and is assumed to be bounded as $\|b_t\|_2 \leq \xi$ for all $t$. This section explores the structure of $b_t$ and how the layer-specific reweighting coefficients $w^{(l)}_t$ influence it, with the goal of minimizing the bias upper bound $\xi$.

\subsection{Optimization Problem}

We compute $\mathbb{E}[\hat{g}_{t,i}^{(l)}]$ with respect to the private training set $D$ as:

\begin{align*}
\mathbb{E}[\hat{g}_{t,i}^{(l)}] = \mathbb{E}_{\mathbf{x}_i \sim D}\left[   C_{t,i} \frac{w_t^{(l)}g_{t,i}^{(l)}}{\|g_{t,i}^{(l)}\|_2} \right] = w^{(l)}_t \cdot \left[ \frac{1}{N} \sum_{\textbf{x}_i \in D} C_{t,i} \frac{g_{t,i}^{(l)}}{\|g_{t,i}^{(l)}\|_2} \right].
\end{align*}
We define:
\begin{equation} \label{eq:appendix_u_def}
u_t^{(l)} = \mathbb{E}_{\textbf{x}_i \sim D} \left[ C_{t,i} \frac{g_{t,i}^{(l)}}{\|g_{t,i}^{(l)}\|_2} \right]=\left[ \frac{1}{N} \sum_{\textbf{x}_i \in D} C_{t,i} \frac{g_{t,i}^{(l)}}{\|g_{t,i}^{(l)}\|_2} \right],
\end{equation}
whose norm can be bounded using Jensen's inequality:
\begin{align*}
\|u_t^{(l)}\|_2 &= \left\| \mathbb{E}_{\textbf{x}_i \sim D} \left[ C_{t,i} \frac{g_{t,i}^{(l)}}{\|g_{t,i}^{(l)}\|_2} \right] \right\|_2
\le \mathbb{E}_{\textbf{x}_i \sim D} \left\| C_{t,i} \frac{g_{t,i}^{(l)}}{\|g_{t,i}^{(l)}\|_2} \right\|_2 \\&= \mathbb{E}_{\textbf{x}_i \sim D} \left[ C_{t,i} \left\| \frac{g_{t,i}^{(l)}}{\|g_{t,i}^{(l)}\|_2} \right\|_2 \right] = \mathbb{E}_{\textbf{x}_i \sim D} [C_{t,i}]\leq C.
\end{align*}
So the expected layer-wise reweighted-and-clipped gradient is:
\[
\mathbb{E}[\hat{G}_{t,i}] = [w^{(1)}_t u_t^{(1)}, \dots, w^{(L)}_t u_t^{(L)}].
\]
The bias term $b_t = \mathbb{E}[\hat{G}_{t,i}] - \nabla F(W_t)$ can be expressed layer-wise. We decompose $\nabla F(W_t)$ as $\nabla F(W_t) = [\nabla F^{(1)}(W_t), \dots, \nabla F^{(L)}(W_t)]$ with $\nabla F^{(l)}(W_t) = \mathbb{E}_{\textbf{x}_i \sim D}[g_{t,i}^{(l)}]$ the layer-$l$ component. Then:
\[
b_t = [w^{(1)}_t u_t^{(1)} - \nabla F^{(1)}(W_t), \dots, w^{(L)}_t u_t^{(L)} - \nabla F^{(L)}(W_t)].
\]
The squared Euclidean norm of the bias is then:
\begin{equation} \label{eq:appendix_bias_norm_sq}
\|b_t\|_2^2 = \sum_{l=1}^L \left\| w^{(l)}_t u_t^{(l)} - \nabla F^{(l)}(W_t) \right\|_2^2.
\end{equation}
The objective is to determine the layer-wise reweighting vector $w_t = [w^{(1)}_t, \dots, w^{(L)}_t]$ subject to $\sum_l (w^{(l)}_t)^2 = 1$, that minimizes the bias norm $\|b_t\|_2$ and thus its upper bound $\xi$. Formally, we solve:
\begin{equation}
    \min_{w_t\in R^L}\|b_t\|_2,\quad \text{s.t.} \sum_l (w^{(l)}_t)^2 = 1.
    \label{eq_solve}
\end{equation}

\subsection{Optimal Reweighting for Bias Minimization}\label{appendix lag}

Expanding Eq.(\ref{eq_solve}), we obtain:
\begin{align*}
    \|b_t\|_2^2= \sum_{l=1}^L \left( (w^{(l)}_t)^2 \|u_t^{(l)}\|_2^2 - 2 w^{(l)}_t (u_t^{(l)}  \nabla F^{(l)}(W_t)) + \|\nabla F^{(l)}(W_t)\|_2^2 \right).
\end{align*}

Let $A_l = \|u_t^{(l)}\|_2^2$ and $B_l = u_t^{(l)} \cdot \nabla F^{(l)}(W_t)$. Disregarding the constant term $\sum_l \|\nabla F^{(l)}(W_t)\|_2^2$, we aim to minimize:
\[
g(\mathbf{w}_t) = \sum_{l=1}^L \left( (w^{(l)}_t)^2 A_l - 2 w^{(l)}_t B_l \right)
\]
subject to the constraint $h(\mathbf{w}_t) = \sum_{l=1}^L (w^{(l)}_t)^2 - 1 = 0$. 

Applying the method of Lagrange multipliers, we define
\[
\mathcal{L}(\mathbf{w}_t, \lambda) =  \sum_{l=1}^L \left( (w^{(l)}_t)^2 A_l - 2 w^{(l)}_t B_l \right) - \lambda \left( \sum_{l=1}^L (w^{(l)}_t)^2 - 1 \right).
\]
Taking the partial derivative with respect to $w^{(l)}_t$ and enforcing stationarity yields $\frac{\partial \mathcal{L}}{\partial w^{(l)}_t} = 2 w^{(l)}_t A_l - 2 B_l - 2 \lambda w^{(l)}_t = 0$. Hence, the optimal solution for each layer is
\begin{equation}
w^{(l)}_t = \frac{B_l}{A_l - \lambda} = \frac{u_t^{(l)} \cdot \nabla F^{(l)}(W_t)}{\|u_t^{(l)}\|_2^2 - \lambda}.  
\label{equation13}
\end{equation}
Substituting Eq.(\ref{equation13}) into the constraint $\sum_{l=1}^L (w^{(l)}_t)^2 = 1$:
\begin{equation} \label{eq:appendix_lambda_eq}
\sum_{l=1}^L \left( \frac{u_t^{(l)} \cdot \nabla F^{(l)}(W_t)}{\|u_t^{(l)}\|_2^2 - \lambda} \right)^2 = 1.
\end{equation}
The Lagrange multiplier $\lambda$ is thus determined by solving Eq.(\eqref{eq:appendix_lambda_eq}), which uniquely specifies the optimal weights $w^{(l)}_t$ that minimize the bias $\|b_t\|_2^2$ under the unit-norm constraint.

Mathematically, solving Eq.(\ref{eq:appendix_lambda_eq}) for $\lambda$ is intricate, as it potentially requires the identification of roots of a high-degree polynomial. In practice, an efficient bisection search can be used to obtain $\lambda$. The resulting $\lambda$ determines the optimal coefficients $w^{(l)}_t$, which minimize $\|b_t\|_2$ under the unit-norm constraint. In our experiments, we benchmark this LM-DP-SGD variant against other DP-SGD methods. To reduce computational overhead, we omit the bias from mini-batch sampling, which is typically negligible in practice. Under this simplification, the expected true gradient of layer $l$ at iteration $t$ can be approximated by the mini-batch average $\nabla F^{(l)}(W_t) \simeq \frac{1}{|\mathcal{B}_t|}\sum_{\mathbf{x}_t\in \mathcal{B}_t} g_{t,i}^{(l)}$, and similarly $u_t^{(l)}\simeq \frac{1}{|\mathcal{B}_t|}\sum_{\mathbf{x}_t\in \mathcal{B}_t}\left[C_{t,i}\frac{g_{t,i}^{(l)}}{\|g_{t,i}^{(l)}\|_2}\right]$. 

\section{Ablation Studies}\label{ablation}

We perform ablation studies of LM-DP-SGD on D-CNN trained with CIFAR10. Specifically, for layer-wise MIA risk estimation, we evaluate the robustness of the assessment under varying shadow dataset distributions. For private training, we analyze the effects of the privacy budget $\varepsilon$, the clipping threshold $C$ and the heterogeneity emphasis factor $r$ on model utility.

\subsection{Impact of Shadow Dataset Distribution}\label{shadow_datset}

In this section, we select shadow datasets whose distributions are either similar to the private training set (ImageNet) or dissimilar (SVHN, MNIST, VGGFace2) to evaluate MIA risk of IRs. For MNIST, each $28\times28$ grayscale image is resized to  $32\times32$ using bilinear interpolation and converting it to 3-channel RGB via channel replication to match the CIFAR10 input format. For VGGFace2, faces are detected and aligned using a standard five-point similarity transform. From the aligned image, we extract a square center crop with a 10\% margin around the face bounding box to preserve facial structure, then downsample to $32\times32$ using bicubic interpolation. To match CIFAR10’s number of classes, we construct an identity-classification shadow task using the ten identities with the largest number of samples. As reported in Table \ref{table_shadow}, across different shadow datasets, MIA error rates at different layers remain comparable and consistently decrease with increasing layer depth, indicating that layer-wise MIA risk estimation is robust to shadow dataset distribution. Since subsequent training further amplify layer-wise heterogeneity to highlight differences in layer risk, the influence of shadow-dataset selection on the private training can be safely disregarded.

\begin{table}[!htp]
\small
\caption{Layer-wise MIA error rates across shadow datasets.}
    \centering
    \begin{tabular}{ccc}
    \toprule
    Shadow Dataset  & Index of Intermediate Convolutional-Layer & Layer-wise MIA Error Rates (\%)\\
    \midrule
         ImageNet &\multirow{4}{*}{1, 3, 5, 7, 8} & 40.4, 35.1, 31.7, 28.4, 26.9\\
         SVHN &&40.2, 35.0, 31.4, 28.1, 26.7\\
         MNIST && 39.9, 34.8, 31.4, 28.1, 26.5\\
         VGGFace2 &&40.8, 35.5, 31.9, 28.7, 27.2\\
    \bottomrule
    \end{tabular}
    \label{table_shadow}
\end{table}

\subsection{Impact of Privacy Budget}

A smaller $\varepsilon$ signifies stronger privacy guarantees, and thus a larger noise scale, which degrades utility and training stability. As shown in Fig. \ref{1(a)}, test accuracy consistently declines at small $\varepsilon$ values, as heavier perturbations obscure true gradient directions and hinder effective model convergence; excessively small $\varepsilon$ can even induce numerical instability, leading to substantial drops in test accuracy (e.g., $\varepsilon=3.0$). Conversely, increasing $\varepsilon$ reduces the noise magnitude, leading to significant improvements in test accuracy. Notably, beyond a certain threshold (e.g., $\varepsilon \geq 8.0$), test accuracy plateaus with no further substantive utility improvements. Correspondingly, the $\ell_2$-norm of the bias $b_t$ increases as $\varepsilon$ decreases: larger noise results in more gradients with $\ell_2$-norms exceeding the fixed clipping threshold, amplifying clipping bias, and, under extreme noise, propagating instabilities that produce invalid (NaN) gradients and model collapse (e.g., $\varepsilon=3.0$). Increasing $\varepsilon$ mitigates these effects by lowering the noise scale, reducing layer-wise reweighted clipping bias, and enhancing gradient fidelity—culminating in marked accuracy gains.

\begin{figure}[ht]
  \centering
  \subfigure[Test accuracy.]{
		\includegraphics[width=0.40\textwidth]{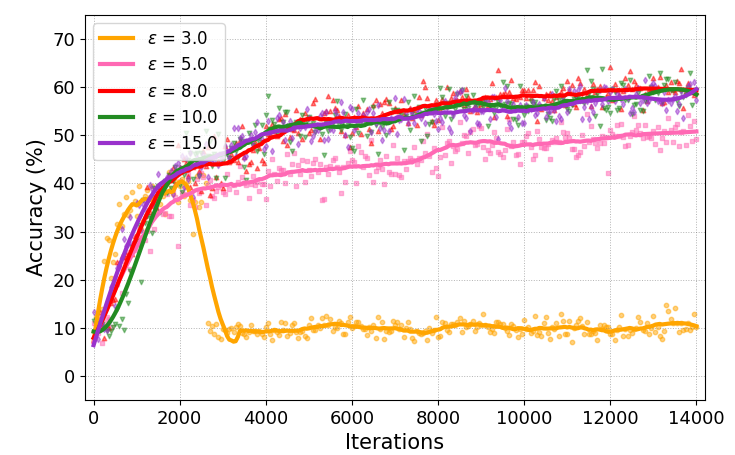}
		\label{1(a)}
    }
  \subfigure[$\ell_2$-norm of $b_t$.]{
		\includegraphics[width=0.40\textwidth]{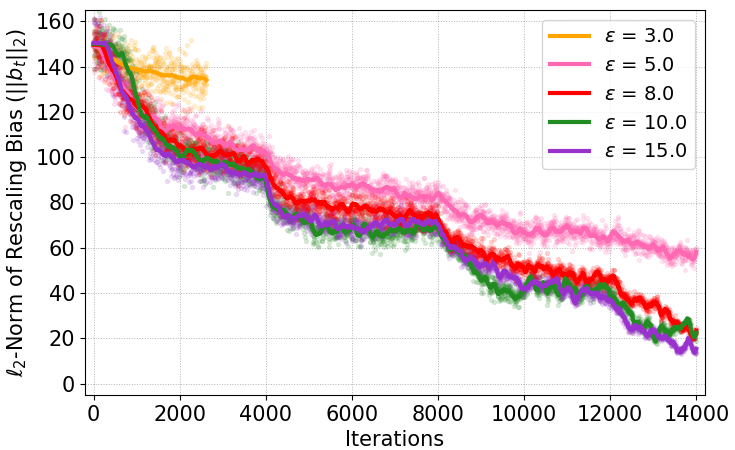}
		\label{1(b)}
    }
    \caption{Impact of privacy budget $\varepsilon$.} 
  \label{privacy budget pb1}
\end{figure}

\subsection{Impact of Clipping Threshold}

This section presents an ablation study on the clipping threshold $C$. As shown in Figure \ref{2(a)}, small values of $C$ impose overly restrictive constraints on gradient norms, reducing update magnitudes and slowing optimization. Such behavior may trap the optimizer in suboptimal minima, resulting in degraded test accuracy. In contrast, larger $C$ values relax gradient constraints; after layer-wise reweighted clipping, gradients with smaller norms better retain their global magnitudes, improving gradient fidelity and enabling faster convergence to higher test accuracy. Figure \ref{2(b)} further illustrates a negative correlation between $C$ and $\ell_2$-norm of $b_t$. The pronounced bias observed at small $C$ stems from two sources: (i) a higher proportion of gradients exceeding the threshold, which leads to more frequent clipping; and (ii) a stronger per-gradient $\ell_2$-norm reduction.

\begin{figure}[ht]
  \centering
  \subfigure[Test accuracy.]{
		\includegraphics[width=0.40\textwidth]{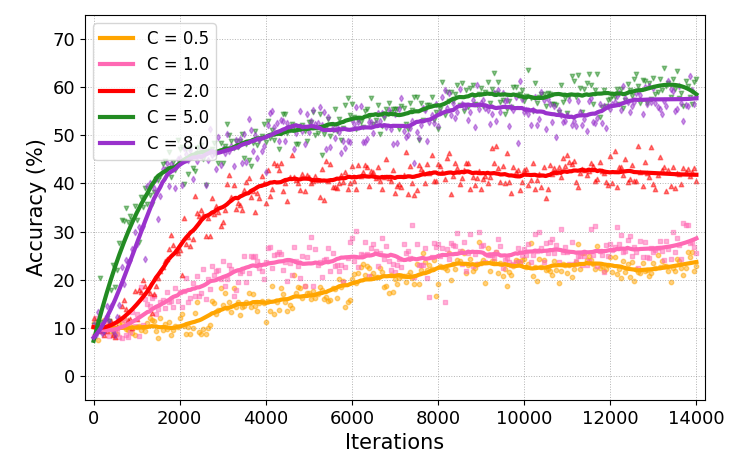}
		\label{2(a)}
    }
  \subfigure[$\ell_2$-norm of $b_t$.]{
		\includegraphics[width=0.40\textwidth]{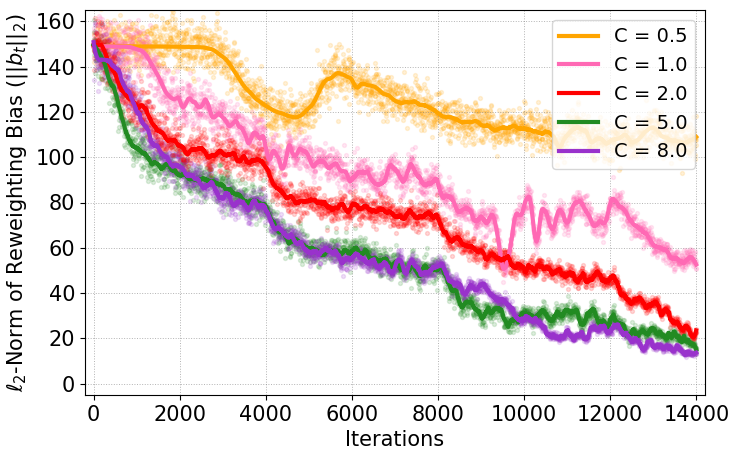}
		\label{2(b)}
    }
    \caption{Impact of clipping threshold $C$.} 
  \label{privacy budget pb2}
\end{figure}

\subsection{Impact of Layer-wise MIA-Risk Heterogeneity Emphasis Factor}

This section examines how the layer-wise MIA-risk heterogeneity emphasis factor $r$ affects model learning dynamics and resilience of IRs against MIAs. As depicted in Figure \ref{privacy budget pb3}, a larger $r$ accentuates the heterogeneity in layer-wise vulnerability to MIAs. This results in reduced model utility, primarily because higher $r$ values intensify the protection of more sensitive layers by shrinking their $\ell_2$-norms. These layers are typically closer to the output, thus diminishing the model’s predictive accuracy. Additionally, increasing $r$ enlarges $\|b_t\|_2$, since stronger layer-wise reweighting accentuates gradient-norm disparities, introducing greater bias through layer-wise reweighted clipping. Furthermore, Table \ref{tablerrrr} quantitatively confirms that increasing $r$ strengthens privacy protection in deeper, more sensitive layers, evidenced by lower MIA accuracy. However, this improvement comes at the cost of weaker privacy in shallower layers, which exhibit heightened vulnerability to MIAs.

In summary, while a larger $r$ bolsters the privacy of highly sensitive layers, it simultaneously compromises model utility. Thus, practical deployment requires careful tuning of $r$ to balance privacy preservation and model performance.

\begin{figure}[ht]
  \centering
  \subfigure[Test accuracy.]{
		\includegraphics[width=0.40\textwidth]{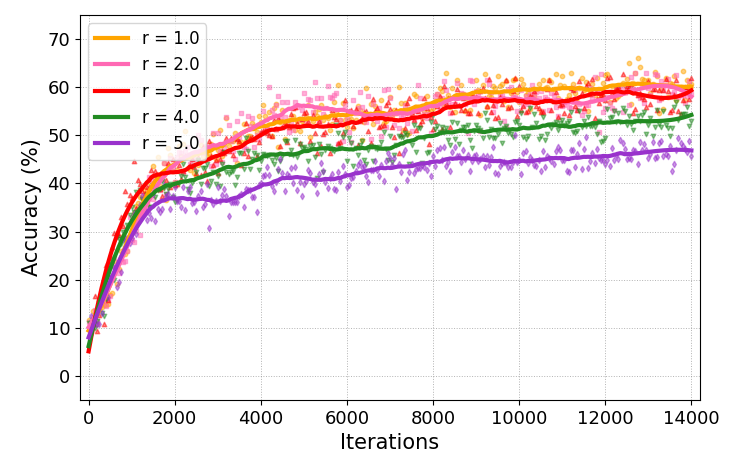}
    }
  \subfigure[$\ell_2$-norm of $b_t$.]{
		\includegraphics[width=0.40\textwidth]{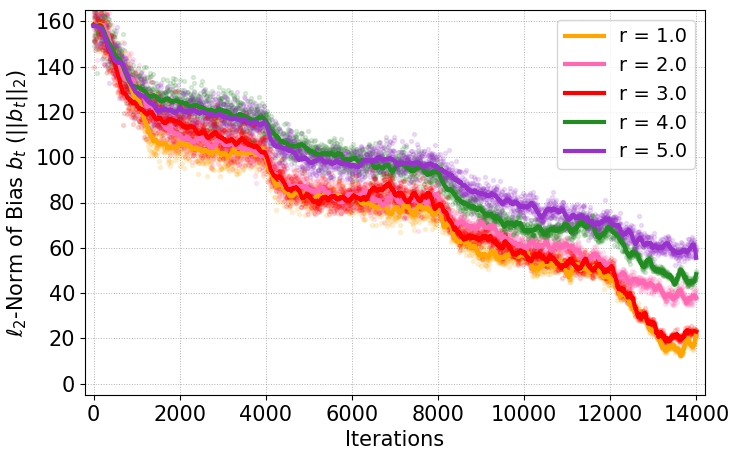}
    }
    \caption{Impact of heterogeneity emphasis factor $r$.} 
  \label{privacy budget pb3}
\end{figure}

\begin{table}[!htp]
\small
\caption{MIA accuracy under varying $r$ values.}
    \centering
    \begin{tabular}{ccc}
    \toprule
    Heterogeneity Emphasis Factor $r$  & Index of Intermediate Convolutional-Layer & MIA Accuracy under LM-DP-SGD (\%)\\
    \midrule
         1.0 &\multirow{5}{*}{1, 3, 5, 7, 8} & 57.4, 63.7, 65.5, 67.5, 69.4\\
         2.0 &&57.8, 63.9, 65.4, 67.2, 69.1\\
         3.0 &&59.0, 64.3, 65.2, 66.5, 67.9\\
         4.0 &&60.3, 64.8, 64.9, 65.8, 67.2\\
         5.0 &&60.9, 65.2, 64.6, 65.4, 66.9\\
    \bottomrule
    \end{tabular}
    \label{tablerrrr}
\end{table}

\section{Running Time and Resources}\label{runningtime}

Tables~\ref{running_time} and \ref{gpu_usage} present the average per-iteration runtime (in seconds) and the peak GPU memory utilization (in MiB) of different DP-SGD frameworks during private training. Compared to baselines that perform direct clipping of global per-example gradients, LM-DP-SGD incurs additional overhead from layer-wise reweighted clipping in both computation and memory. Empirically, the runtime remains comparable to baselines, and the modest increase in GPU utilization is limited and remains practically acceptable.

\begin{table}[!ht]
    \centering
    \caption{Average running time (seconds) per iteration.}
    \resizebox{0.65\textwidth}{!}{
    \begin{tabular}{cccccc}
    \hline
    \toprule
        Dataset & Model & DP-SGD & Auto-S/NSGD & DP-PSAC & LM-DP-SGD \\
    \midrule
             MNIST & S-CNN& 0.0984 & 0.1071 &  0.1049 &0.1078\\
             CIFAR10 & D-CNN  & 0.1516 & 0.1569 & 0.1536 &0.1632\\
             CIFAR100 & ResNet-18 & 0.2735 & 0.2845 & 0.2857 &0.2860\\ 
             CelebA & VGG-16  & 0.2557 & 0.2593 &  0.2615&0.2654\\
             
    \bottomrule
    \end{tabular}}
    \label{running_time}
    \end{table}

    \begin{table}[!ht]
    \centering
    \caption{Peak GPU memory utilization (MiB).}
    \resizebox{0.65\textwidth}{!}{
    \begin{tabular}{cccccc}
    \hline
    \toprule
        Dataset & Model & DP-SGD & Auto-S/NSGD & DP-PSAC & LM-DP-SGD \\
    \midrule
             MNIST & S-CNN& 2342 & 3360 & 3526 & 3734\\
             CIFAR10 & D-CNN  & 14524 & 15498 &  15926 &19918 \\
             CIFAR100 & ResNet-18 & 40378 & 41295 & 413452 &46578\\ 
             CelebA & VGG-16 & 36304 & 37459 & 37857 &40352 \\
             
    \bottomrule
    \end{tabular}}
    \label{gpu_usage}
\end{table}

\end{document}